\documentclass[11pt]{article}

\usepackage{fix-cm}
\usepackage[utf8]{inputenc}
\usepackage[margin=1in]{geometry}

\usepackage[utf8]{inputenc}
\usepackage{fancyhdr}
\usepackage{lastpage}
\usepackage{amsfonts}
\usepackage{amsmath}
\usepackage{amssymb}
\usepackage{bm}
\usepackage{bbm} 
\usepackage{url}
\usepackage{graphicx}
\usepackage{color}
\usepackage{multirow}
\usepackage{bbm}
\usepackage{array}
\usepackage{mathtools}
\usepackage{xcolor}
\usepackage[dvipsnames]{xcolor}
\newcommand\inprod[1]{\langle #1 \rangle}
\newcommand{\norm}[1]{\left \| #1 \right \|}
\definecolor{rose}{RGB}{180,60,80}

\newcommand{\simp}{\Delta_d^\circ}
\newcommand{\dom}{\mathcal{W}}
\newcommand{\w}{w}
\newcommand{\initbeta}{\beta(w_0, \epsilon)}
\usepackage{algorithm}
\usepackage{algpseudocode}

\usepackage{wrapfig}
\usepackage{caption}

\usepackage{booktabs}
\usepackage{tabularx}
\usepackage{makecell}

\usepackage{tikz}
\usetikzlibrary{arrows.meta}
\usetikzlibrary{calc}

\newcommand{\dual}[1]{\| #1 \|_*}

\usepackage[round]{natbib}

\usepackage{csquotes}

\usepackage[shortlabels]{enumitem}

\usepackage[english]{babel}

\usepackage[colorlinks=true, allcolors=blue]{hyperref}

\usepackage{amsthm}
\usepackage{aliascnt}
\usepackage[noabbrev,capitalise]{cleveref}
\newtheorem{theorem}{Theorem}

\newtheorem{definition}{Definition}

\newaliascnt{lemma}{theorem}
\newtheorem{lemma}[lemma]{Lemma}
\aliascntresetthe{lemma}

\crefname{lemma}{Lemma}{Lemmas}
\Crefname{lemma}{Lemma}{Lemmas}
\newcommand{\samethanks}[1][\value{footnote}]{\footnotemark[#1]}

\title{Mirror Descent Beyond Euclidean Stability: \\
An Exponential Separation in Initialization Sensitivity}

\makeatletter
\renewcommand{\@fnsymbol}[1]{%
  \ensuremath{%
    \ifcase#1
    \or \dagger
    \or \ddagger
    \or \mathsection
    \or \mathparagraph
    \or \|
    \else\@ctrerr
    \fi
  }%
}
\makeatother

\author{%
    Shira Vansover-Hager%
    \thanks{\raggedright Blavatnik School of Computer Science and AI, Tel Aviv University; 
    \texttt{\char`\{shirav,schliserman,ofirs4\char`\}@mail.tau.ac.il}.}
    \and%
    Matan Schliserman\samethanks[1]
    \and%
    Ofir Schlisselberg\samethanks[1]
    \and%
    Tomer Koren%
    \thanks{\raggedright Blavatnik School of Computer Science and AI, Tel Aviv University, and Google Research; 
    \texttt{tkoren@tauex.tau.ac.il}.}
}
\date{}

\newcommand{\R}{\mathbb{R}}
\renewcommand{\epsilon}{\varepsilon}

\begin{document}

\maketitle

\begin{abstract}
Mirror Descent (MD) extends Gradient Descent (GD) beyond Euclidean geometry and
has recently reappeared as a lens for KL-regularized policy optimization in
reinforcement learning and LLM post-training. This raises a basic robustness question, 
crucial to reproducibility and reliability: how sensitive are MD dynamics to their inputs? We focus on initialization, often itself a pretrained or previously aligned model.
Quadratic-regularized MD, including GD and Mahalanobis geometries, is well-known to be stable
for convex smooth objectives. We show a sharp contrast: once the regularizer is
non-quadratic, MD can be exponentially more sensitive to initialization than
GD, even with a well-conditioned regularizer in Euclidean norm.
We give a three-dimensional construction with a convex, smooth
objective and a strongly convex, smooth, well-conditioned regularizer where an initial $\varepsilon$ perturbation is quickly amplified  to
$
    \smash{
    \min\{
        \text{polylog}^{-1}(1/\varepsilon),
        \varepsilon e^{\Omega(\eta T)}
    \}
    }
$
after $T$ iterations of MD with step size $\eta$.
For canonical KL-regularized MD on the simplex, we show that even linear objectives can amplify an
initial $\varepsilon$ perturbation exponentially fast in high-dimensional or near-boundary regimes.
Finally, we show that adding a Bregman regularization term toward an anchor point can stabilize the dynamics while largely preserving the optimization guarantees, and that the choice of anchor is crucial: anchoring at the initialization only partially mitigates the instability, whereas anchoring at a fixed point yields a more stable mechanism.
\end{abstract}

\allowdisplaybreaks
\section{Introduction}\label{sec:introduction}
Mirror Descent (MD) \citep{nemirovskii1983problem,beck2003mirror} is a
fundamental optimization paradigm that adapts its updates to the geometry of
the parameter space. Its updates, typically written as
\[
w_{t+1}=
\arg\min_{w}
\left\{
\eta \langle \nabla F(w_t), w \rangle
+
D_R(w , w_t)
\right\},
\]
where \(D_R\) is the Bregman divergence induced by a regularizer \(R\) and
\(\eta\) is the step size, optimize a local linearization of the objective
while penalizing deviations according to the geometry specified by \(R\).
Notable instances include Gradient Descent
\citep{nesterov1998introductory} and multiplicative weights updates
\citep{littlestone1994weighted,freund1997decision,arora2012multiplicative}.

The stability of these dynamics is increasingly important in modern machine
learning. MD has reappeared as a useful lens for KL-regularized policy
optimization in reinforcement learning
\citep{schulman2015high,schulman2017proximal,akkaya2019solving} and LLM
post-training \citep{ouyang2022training,shao2024deepseekmath}. In such
settings, optimization is often initialized from a pretrained, supervised
fine-tuned, or otherwise aligned model. Changes in pretraining data,
randomness, or checkpoint choice therefore perturb the initialization itself,
and potentially the final model. 
This raises a basic robustness question, central to reproducibility and
reliability: how much can mirror-descent dynamics amplify a small perturbation
in their initialization after $T$ sequential steps? We study this question
through the lens of \emph{initialization stability}: the worst-case change in
the output of the \(T\)-step algorithm under an \(\varepsilon\)-perturbation of
its starting point. Equivalently, this is a local worst-case robustness, or
finite-time sensitivity, notion for the algorithm's output as a function of the
initialization.

For quadratic regularization in Euclidean geometry, the answer is well known. 
This class includes Gradient Descent
and Mahalanobis geometries, and for convex smooth objectives these algorithms
are known to be extremely stable: small initialization perturbations remain small along
the trajectory and at the final output \citep{hardt2016train}. Much less is
known for non-Euclidean mirror maps. A key difference from the quadratic-regularized case  
relates to the conditioning
of the regularizer: if the mirror map is poorly conditioned, a small primal
perturbation can correspond to a much larger displacement in the geometry used
by the update. Conversely, one might hope that uniformly well-conditioned
mirror maps behave similarly to quadratic ones and retain their favorable stability. 

The central finding of this paper is that this hope is false: even for convex smooth objectives,
\emph{general mirror descent can be exponentially more sensitive to initialization as compared to Gradient Descent}.
We establish this phenomenon in two complementary regimes, separating the role of nonquadratic geometry from the additional effects of ill-conditioning. First, we show that exponential instability is not merely an artifact of a poorly conditioned regularizer. In three-dimensional Euclidean space equipped
with the usual \(\ell_2\) norm, we construct a convex, smooth objective and a strongly convex, smooth, well-conditioned (nonquadratic) regularizer, all with respect to this norm, for which an \(\varepsilon\)-perturbation of the initialization is
amplified by a factor \(\smash{e^{\Omega(\eta T)}}\), up to $\smash{\widetilde\Theta(1)}$ saturation scale. This phenomenon can arise even for step-size choices for which MD
is guaranteed to optimize the objective. This implies an exponential separation from the stability enjoyed 
by quadratic regularization and standard Gradient Descent.

Second, we study the canonical entropic geometry on the simplex, where \(R(w)=\sum_i w_i\log w_i\) and \(D_R\) is the KL divergence. This geometry underlies classical multiplicative-weights methods and modern KL-regularized updates in policy optimization and model post-training. Here the conditioning mechanism is visible more directly: negative entropy is highly ill-conditioned in \(\ell_1\) geometry near low-mass coordinates, and we show that this alone can drive exponential amplification even for linear objectives. The resulting lower bound is uniform over all simplex initializations in the high-dimensional regime, and also captures the low-dimensional, near-boundary regime. Notably, this instability can arise even for step-size ranges in which MD still converges to a minimizer. We complement this lower bound with a matching exponential upper bound for entropy MD, and with an extension to more general Legendre regularizations frequent in online optimization and RL~\citep{cesa2006prediction}.

Finally, we ask whether this instability of MD can be mitigated without abandoning the geometry and generality of the method. For this, we introduce two variants of MD that stabilize the algorithm by anchoring it to a reference point through an additional Bregman regularization term.
First, motivated by practical settings such as KL-regularized fine-tuning, where the optimization process is initialized from a pretrained model that also serves as the reference point \citep[e.g.,][]{ouyang2022training,shao2024deepseekmath}, we study \emph{Initialization-Anchored MD}. In this method, the additional regularization term is given by the Bregman distance \emph{to the initialization point}. We show that this variant stabilizes MD in the well-conditioned setting, assuming the regularizer is also smooth, achieving initialization stability of \(O(\varepsilon + 1/\sqrt{T\log T})\) together with optimization error \(O(\log T/T)\). However, in ill-conditioned settings, the guarantee of this variant may become vacuous due to its dependence on local smoothness of the regularizer at initialization. To overcome this, we introduce a second variant, \emph{Fixed-Anchor MD}, in which the Bregman regularization is anchored at a \emph{fixed reference point} independent of the initialization. We prove that this method remains stable even for ill-conditioned regularizers, achieving \(O(1/T)\) initialization stability while preserving the optimization guarantees of MD up to logarithmic factors, with optimization error \(O(\log (T)/T)\). These results extend the regularization-based stabilization perspective of \citet{attia2022uniform} from uniform stability to initialization stability.

Taken together, our results suggest that initialization sensitivity should be treated as a
primary consideration when MD is used as a modeling abstraction for modern
optimization pipelines. In KL-regularized policy optimization or LLM
post-training, small differences in the starting reference model can be quickly amplified
through only a few sequential updates, and this may occur even if the initial model has significant entropy. At the same time, the Bergman-regularized
algorithms indicate that this sensitivity is not inevitable: stability can be
improved by augmenting with additional regularization in the same geometry, preserving optimization rates (up to log factors).
\subsection{Summary of contributions}

In more detail, our main contributions in this paper are as follows.
\begin{itemize}[leftmargin=!]
\item 
We show that MD exhibits exponential initialization sensitivity already in a
low-dimensional, well-conditioned Euclidean geometry. Specifically, in
dimension \(d=3\), with respect to the standard \(\ell_2\) norm, we construct a convex smooth
objective and a strongly convex, smooth, well-conditioned nonquadratic
regularizer for which MD has initialization instability
\[
    \Omega\!\left(
    \min\left\{
        \operatorname{polylog}^{-1}(1/\varepsilon),
        \varepsilon e^{\Omega(\eta T)}
    \right\}\right).
\]
This gives an exponential separation from quadratic regularization: for
quadratic MD, including Gradient Descent and Mahalanobis geometries,
initialization perturbations remain bounded by \(O((\beta/\alpha)\varepsilon)\) 
throughout the algorithm's trajectory.

\item 
For the canonical entropic/KL geometry on the simplex, we give a sharp
characterization of the initialization stability of MD, which is again
exponential in \(\eta T\). We prove that negative-entropy MD can amplify an
\(\varepsilon\)-perturbation by
\(\Omega(\min\{1,\varepsilon e^{\eta T}\})\), even for linear objectives. In the
high-dimensional regime \(d\ge1/\varepsilon\), this result holds uniformly over
all initializations; the same result also captures near-boundary worst-case
initializations in low dimension. We complement it with a matching exponential
upper bound for entropy MD and an extension to Legendre
regularizers that are central in online optimization and RL.

We present two Bregman-regularized variants of MD for mitigating initialization instability. The first, \emph{Initialization-Anchored MD}, adds a Bregman regularization term centered at the initialization. In the well-conditioned setting, where the regularizer is also smooth, it achieves initialization stability \(O(\varepsilon + 1/\sqrt{T\log T})\) and optimization error \(O(\log (T) / T)\), however, its guarantees may become vacuous for ill-conditioned regularizers. The second variant, \emph{Fixed-Anchor MD}, adds a Bregman regularization term centered at a fixed reference point independent of the initialization. This method also handles ill-conditioned regularizers, achieving initialization stability \(O(1/T)\) while preserving the convergence guarantees of MD up to logarithmic factors, with optimization error \(O(\log (T) / T)\).
\end{itemize}

\newcolumntype{C}[1]{>{\centering\arraybackslash}p{#1}}

\begin{figure}[t]
\centering
\scriptsize
\setlength{\tabcolsep}{2pt}
\renewcommand{\arraystretch}{1.2}

\begin{tabularx}{\textwidth}{%
    C{2cm}
    C{1.5cm}
    C{1.5cm}
    C{2cm}
    C{1.5cm}
    C{5cm}
    C{2.0cm}
}
\toprule
\textbf{Algorithm}
& \textbf{Type}
& \textbf{Domain}
& \textbf{Regularizer}
& \textbf{Assumptions}
& \textbf{Bound}
& \textbf{Reference}
\\
\midrule

MD
& Upper
& Convex
& Quadratic, $\kappa$-conditioned
& $\eta \leq \frac{\alpha}{L}$
& $O \left(\frac{\beta}{\alpha}\,\varepsilon\right)$
& \citet{hardt2016train} (see~\cref{thm:quad-upper})
\\

\specialrule{1.1pt}{1.5pt}{1.5pt}

MD
& Lower
& Convex
& Euclidean, $\kappa = O(1)$, nonquadratic
& --
& $\Omega(\min\!\left\{
        \operatorname{polylog}^{-1}(\frac{1}{\varepsilon}),
        \varepsilon e^{\Omega(\eta T)}
    \right\})$
& \cref{thm:endpoint_exponential_origin_local_feasible_set}
\\

\specialrule{1.1pt}{1.5pt}{1.5pt}

MD
& Lower
& Simplex
& Negative entropy
& $d \geq \frac{1}{\varepsilon}$ or $w_0^{\min} \leq \varepsilon$
& $\Omega\!\left(\varepsilon e^{\eta T}\right)$
& \cref{thm:kl-lower}
\\

\midrule

MD
& Upper
& Simplex
& Negative entropy
& --
& $O\!\left(\varepsilon e^{O(\eta T)}\right)$
& \cref{thm:kl-upper-bound}
\\

\specialrule{1.1pt}{1.5pt}{1.5pt}

Init.-Anchor MD
(\cref{alg:anchored-MD-init})
& Upper
& Convex
&$\kappa = O(1)$
& --
& $O\!\left(\epsilon + 1/\sqrt{T\log T}\right)$
& \cref{thm:stabilizing-alg-init}
\\
\midrule
Fixed-Anchor MD
(\cref{alg:anchored-MD-fixed})
& Upper
& Convex
&--
& --
& $O\!\left(1/T\right)$
& \cref{thm:stabilizing-alg-fixed}
\\

\bottomrule
\end{tabularx}
\caption{Summary of initialization-stability bounds. Here $w_0$ denotes the initialization, $w_0^{\min}$ the minimal coordinate of $w_0$, $\varepsilon$ denotes the initialization perturbation, $T$ the optimization horizon, $\eta$ the step size, $L$ the smoothness parameter of the objective, and $\alpha,\beta$ the strong-convexity and smoothness parameters of the regularizer and \(\kappa = \beta/\alpha\). 
}
\label{fig:summary-table}
\end{figure}

\subsection{Related work}
\paragraph{Mirror descent and non-Euclidean optimization.}
Mirror Descent has been central to optimization and online learning for several
decades; see, e.g.,
\citet{shalev2025online,bubeck2015convex,hazan2016introduction,beck2017first}
for textbook and survey treatments. Recent work continues to refine its
optimization and regret guarantees, including stochastic MD for relatively
smooth objectives \citep{d2021stochastic} and online MD with approximate
updates \citep{schlisselberg2025hidden}. Other works study the implicit bias of
MD: in stochastic overparameterized problems, \citet{azizan2021stochastic}
show that MD converges to a global minimizer that is approximately closest to
the initialization in Bregman divergence, while
\citet{sun2022mirror,sun2023unified} characterize its max-margin bias for
linearly separable classification. Our focus is different: we study the
dynamical sensitivity of MD to perturbations in its initialization.

\paragraph{Stability of optimization algorithms.}
Most stability analyses concern \emph{uniform stability}
\citep{bousquet2002stability}, where two runs differ in one component of a
finite-sum objective. In this setting, \citet{shalev2010learnability} proved
stability of strongly convex empirical risk minimization, and
\citet{hardt2016train,lei2020fine} extended stability guarantees to GD and SGD
under smoothness assumptions. These works also imply favorable initialization
stability for Euclidean gradient methods. We show that this behavior does not
extend to general MD: even for convex smooth objectives, nonquadratic mirror
maps can exponentially amplify initialization perturbations.

\paragraph{Lower bounds and noncontractivity.}
Several works show that stability can fail for algorithms outside the basic GD
template. For nonsmooth objectives, \citet{bassily2020stability} prove lower
bounds for GD and SGD. In smooth settings, \citet{attia2021algorithmic} show
that accelerated methods can exhibit exponential initialization instability,
and \citet{schliserman2025flat} prove polynomial instability bounds for
sharpness-aware methods such as SAM. These algorithms are not standard MD.
Closer to our setting, \citet{asi2021private} show that MD can be
non-contractive in a single update. We prove a stronger dynamical statement:
the expansion can persist for \(T\) steps and become exponential in \(T\),
already under convex smooth objectives.

\paragraph{Reproducibility and stabilization.}
A related line of work studies reproducibility under noisy or inexact
operations, including inexact initialization. \citet{ahn2022reproducibility}
prove upper and lower bounds for variants of GD and SGD in convex
optimization, and \citet{zhang2023optimal} propose a black-box stabilization
reduction based on Euclidean regularization. Another line of work shows that
additional regularization can improve uniform stability without degrading
optimization guarantees \citep{attia2022uniform,vary2024black}. Our lower
bounds show that non-Euclidean MD has qualitatively different initialization
sensitivity, and our regularized algorithm extends the regularization-based
stabilization perspective to general mirror maps and initialization stability.

\section{Preliminaries}

\paragraph{Mirror Descent.}

We study the general mirror descent (MD) method \citep{nemirovskii1983problem,beck2003mirror} in $\R^d$ equipped with a norm $\|\cdot\|$. Let \(\|x\|_* = \sup_{\|y\| \leq 1} \langle x, y \rangle\) denote the dual norm.
Given a regularization function \(R : \dom \to \mathbb{R}\), and starting from initialization $w_0 \in \dom$, the mirror descent method takes updates of the form 
\[
    w_{t+1} = \arg\min_{w \in \dom} \,
    \bigl\{\, \eta \langle \nabla F(w_t), w \rangle + D_{R}(w, w_t) \,\bigr\},
\]
where \(F : \dom \to \mathbb{R}\) is the objective function, $\eta>0$ is the step-size, and
\(
D_R(w, w_t) = R(w) - R(w_t) - \langle \nabla R(w_t), w - w_t \rangle
\)
denotes the Bregman divergence induced by \(R\).

As is standard in mirror descent analyses, we assume that $R$ is \(\alpha\)-strongly convex with respect to the norm \(\|\cdot\|\).%
\footnote{A function $H:\dom\to \R$ is $\alpha$-strongly convex with respect to \(\|\cdot\|\) if 
$
H(w) \geq H(w') + \nabla H(w')^\top (w - w') + \frac{\mu}{2} \|w - w'\|^2
$ for all $w,w' \in \dom$.} 
We further assume that the objective \(F\) is convex, \(G\)-Lipschitz,%
\footnote{
A function $H:\dom\to \R$  is \(G\)-Lipschitz with respect to \(\|\cdot\|\) if 
\(|F(w) - F(w')| \leq G \|w - w'\|\) for all \(w,w' \in \dom\).
} 
and \(L\)-smooth,%
\footnote{
A function $H:\dom\to \R$  is \(L\)-smooth with respect to \(\|\cdot\|\) if 
\(\|\nabla F(w) - \nabla F(w')\|_* \leq L \|w - w'\|\) for all \(w,w' \in \dom\).
} both with respect to the norm \(\|\cdot\|\),
and that the domain $\dom$ is compact with diameter $D=\sup_{u,v\in \dom} \|u-v\|$.

\paragraph{Initialization stability.}

We study the stability of mirror descent with respect to its initialization point.
The notion of initialization stability \citep{attia2021algorithmic} measures the worst-case sensitivity of the algorithm's output to a small perturbation in its initial point; formally,
\begin{definition}[Initialization stability]
Let $A$ be an algorithm that given an initialization $w_0$ produces $A(w_0)$ as an output. 
For any $\epsilon>0$, $\epsilon$-initialization stability of $A$ at $w_0\in \dom$ is given by
\[
    \delta_A(w_0, \epsilon)
    \,:=\,
    \sup_{\substack{p :\, \|p\|\leq \epsilon,\\ w_0+p\in\dom}} 
    \bigl\| A(w_0)-A(w_0+p) \bigr\|
    .
\]
\end{definition}
This is a local worst-case sensitivity notion for the algorithmic map
\(w_0\mapsto A(w_0)\). Thus initialization stability can be viewed as adversarial robustness, or equivalently finite-time sensitivity to initial conditions, with respect to worst-case perturbations of the starting point.

\paragraph{Conditioning of the mirror map.}

The geometry induced by \(R\) is controlled by its strong convexity and
smoothness with respect to the norm $\|\cdot\|$. When $R$ is $\beta$-smooth in addition to being $\alpha$-strongly convex, we define its \emph{condition number} by \(\kappa=\beta/\alpha\). We call the geometry \emph{well-conditioned} when
\(\kappa\) is bounded by an absolute constant, and \emph{ill-conditioned} when this
quantity is large or unbounded.

For ill-conditioned regularizers, we quantify the $\epsilon$-local smoothness of the regularizer \(R\) around a point in the optimization domain. This quantity, specifically at the initial point of MD, is shown to play a central role in the stability analysis of MD.
\begin{definition}[$\epsilon$-local smoothness of regularizer]
For a point \(w_0\in\dom\) and radius \(\epsilon>0\), define the $\epsilon$-local smoothness of \(R\) at \(w_0\) by
\[
    \initbeta
    \,:=\,
    \frac{1}{\epsilon}
    \sup_{\substack{p:\, \|p\|\leq \epsilon,\\ w_0+p\in\dom}}
    \bigl\|
        \nabla R(w_0+p)-\nabla R(w_0)
    \bigr\|_* .
\]
\end{definition}

\section{Instability in Well-conditioned Euclidean Geometries}
\label{sec:euclidean}

We begin by showing that an exponential lower bound on the initialization stability of MD in a well conditioned case in Euclidean geometry, i.e., when the regularizer $R$ has condition number $O(1)$ with respect to the Euclidean norm $\|\cdot\|_2$. 
The following theorem establishes this already in dimension $d=3$, with a convex and smooth objective and domain $\dom \subset \mathbb{R}^3$ contained in a ball of diameter $O(1)$, all with respect to $\|\cdot\|_2$.
Throughout this section and its proof, matrix and multilinear-map norms are the
operator norms induced by the Euclidean norm.

\begin{theorem}
\label{thm:endpoint_exponential_origin_local_feasible_set}
There exist absolute constants \(c,C_0,\eta_0,\varepsilon_0>0, \) such that the following
holds. For every $0 \leq \varepsilon\leq \varepsilon_0$,  every
\(0<\eta\le \eta_0\), and every integer \(T\ge1\), there exist a convex
feasible set \(\mathcal W\subset\mathbb R^3\) with diameter at most \(3\),
a regularizer \(R:\mathcal W\to\mathbb R\), a convex objective
\(F:\mathcal W\to\mathbb R\), and a unit vector \(v\in\mathbb R^3\), such that
\(R\) is \(1\)-strongly convex and \(C_0\)-smooth on \(\mathcal W\), \(F\) is
convex, \(1\)-Lipschitz, and \(1\)-smooth on \(\mathcal W\), such that if $A$ is the MD algorithm with regularization $R$ applied for $T$ steps with step size $\eta$, then
$$\delta_A(0, \varepsilon) = \Omega \left(
    \min\left\{
        \frac1{(1+\log(1/\varepsilon))^3},
        e^{c\eta T} \varepsilon
    \right\}
    \right).
$$
\end{theorem}
In particular, for the canonical choice \(\eta = \Theta(1/\sqrt{T})\), under
which MD attains its standard \(O(1/\sqrt{T})\) optimization guarantee, the
bound is still exponential in $T$ and becomes:
\[
\delta_A(0, \varepsilon)
=
\Omega \left(
    \min\left\{
        \frac1{(1+\log(1/\varepsilon))^3},
        e^{\Theta(\sqrt{T})} \varepsilon
    \right\}
    \right).
\]
Thus, with a non-quadratic regularizer and a non-linear objective, mirror descent can amplify initial errors at rate exponential in $\eta T$, up to a saturation level that depends only poly-logarithmically on~$\varepsilon$. For example, when $T \lesssim ({1}/{\eta}) \log(1/\varepsilon)$, the exponential lower bound is the dominant term; in particular, taking $\varepsilon = 1/T^{\gamma}$ yields exponential growth up to scale $(1+\gamma\log T)^{-3}$. 
This gives an exponential separation from the standard quadratic-regularization case, where initialization stability is known to be bounded as $O(\varepsilon)$ and does not grow with $T$ (for completeness, we provide a proof in \cref{sec:addition-upper}). We note that such exponential amplification is considerably easier to obtain with an ill-conditioned regularizer and an ill-conditioned initialization; see details in \cref{sec:ill-conditioning}. Here we address the more benign well-conditioned setting where reproducing this behavior is more challenging.

We provide next a sketch of the proof. The full construction and proof are deferred to \cref{sec:euclidean-proof}.

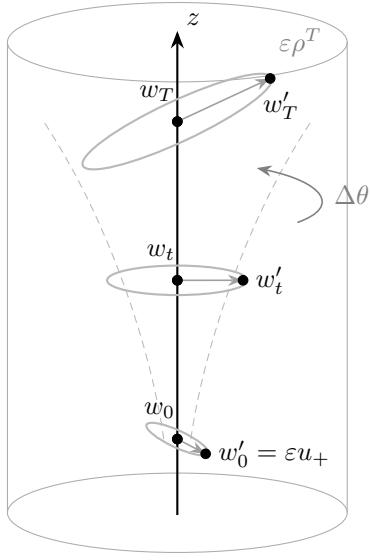
\begin{figure}[t]
\centering
\begin{tikzpicture}[
    scale=1.5,
    >=Stealth,
    every node/.style={font=\small},
    cyl/.style={gray!65},
    guide/.style={gray!55,densely dashed},
    ann/.style={black!50, semithick},
    perturb/.style={semithick,->},
    box/.style={draw=gray!60, rounded corners=2pt, align=center, inner sep=7pt}
]
    \def\a{1.50}
    \def\b{0.35}
    \def\H{4.15}

    \def\zO{0.65}
    \def\zM{2.05}
    \def\zT{3.45}

    \def\ra{0.30}
    \def\rb{0.62}
    \def\rc{0.92}
    \def\ea{0.075}
    \def\eb{0.13}
    \def\ec{0.18}

    \coordinate (wpO) at (0.25,\zO-0.13);
    \coordinate (wpM) at (0.58,\zM);
    \coordinate (wpT) at (0.82,\zT+0.38);

    \draw[cyl] (0,0) ellipse ({\a} and {\b});
    \draw[cyl] (-\a,0) -- (-\a,\H);
    \draw[cyl] (\a,0) -- (\a,\H);
    \draw[cyl] (0,\H) ellipse ({\a} and {\b});

    \draw[black,thick,->] (0,-0.02) -- (0,\H+0.1);
    \node[right] at (-0.01,\H+0.2) {$z$};

    \draw[gray!75,perturb] (0,\zO) -- (wpO);
    \draw[gray!75,perturb] (0,\zM) -- (wpM);
    \draw[gray!75,perturb] (0,\zT) -- (wpT);
    \node[left=-3pt, ann] at (1.3,\zT+0.7) {$\varepsilon\rho^T$};

    \filldraw[black] (0,\zO) circle (1.2pt);

    \filldraw[black] (0,\zM) circle (1.2pt);

    \filldraw[black] (0,\zT) circle (1.2pt);
    \node[right=20pt] at (0.2,\zT+0.14) {$w_T'$};
    \filldraw[black] (0,\zO) circle (1.2pt);
    \node[left=-3pt] at (0,\zO+0.25) {$w_0$};

    \filldraw[black] (0,\zM) circle (1.2pt);
    \node[left=-3pt] at (0,\zM + 0.25) {$w_t$};

    \filldraw[black] (0,\zT) circle (1.2pt);
    \node[left=-3pt] at (0,\zT+0.25) {$w_T$};

    \draw[guide]
        (0.12,\zO)
        .. controls (0.20,1.25) and (0.52,2.45) .. (1.18,\zT);
    \draw[guide]
        (-0.12,\zO)
        .. controls (-0.20,1.25) and (-0.52,2.45) .. (-1.18,\zT);

    \draw[gray!55,thick,rotate around={-25:(0,\zO)}]
        (0,\zO) ellipse ({\ra} and {\ea});
    \draw[gray!55,thick,rotate around={0:(0,\zM)}]
        (0,\zM) ellipse ({\rb} and {\eb});
    \draw[gray!55,thick,rotate around={25:(0,\zT)}]
        (0,\zT) ellipse ({\rc} and {\ec});

    \filldraw[black] (wpO) circle (1.2pt);
    \filldraw[black] (wpM) circle (1.2pt);
    \filldraw[black] (wpT) circle (1.2pt);

    \node[right=1pt] at (wpO) {$w'_0 = \varepsilon u_+$};
    \node[right=1pt] at (wpM) {$w'_t$};

    \coordinate (rotc) at (0.68,2.30);
    \draw[ann,->]
    ($(rotc)+(0.38,0.26)$)
    .. controls ($(rotc)+(0.78,0.40)$) and ($(rotc)+(0.60,0.62)$)
    .. ($(rotc)+(0.02,0.74)$);
    \node[ann] at ($(rotc)+(0.86,0.5)$) {$\Delta\theta$};
\end{tikzpicture}
\caption{
Illustration of the MD trajectories in the proof of~\cref{thm:endpoint_exponential_origin_local_feasible_set}. The trajectory follows the clock coordinate \(z\), while transverse \(x\)-perturbations grow. The transverse Hessian rotates along the path; in the co-moving frame, the dynamics reduce to powers of a fixed matrix \(A\) with expanding eigenvector \(u_+\) and eigenvalue \(\rho>1\).
}
\label{fig:endpoint_exponential_origin_mechanics}
\end{figure}

\textit{Proof sketch.} 
Let $w_0' = w_0 +p$ be a perturbed initialization. The goal is to construct mirror descent dynamics such that, for the trajectory  $\{w_t'\}_{t=0}^T$, the small perturbation
of the initialization grows exponentially relative to a reference
trajectory \(\{w_t\}_{t=0}^T\). We
denote the MD update rule as $w_{t+1}=\Phi(w_t)$, where,
\begin{equation}
\label{eq:full_dynamics}
    \Phi(w):=(\nabla R)^{-1}\bigl(\nabla R(w)-\eta\nabla F(w)\bigr).
\end{equation}
Then, using linear approximation of the updates, we get
\begin{align*}
    w_T' - w_T & = \Phi(w_{T-1}) - \Phi(w_{T-1}) \approx D\Phi(w_T)^\top(w_{T-1}'-w_{T-1}) 
\end{align*}
and, iterating over $t$ we get,
\[
    w_T' - w_T \approx \left(\Pi_{t=1}^T D\Phi(w_t)\right)^\top (w_0' - w_0),
\]
where
 $$D\Phi(w_t)
    =
    \nabla^2R(w_{t+1})^{-1}
    \bigl(\nabla^2R(w_t)-\eta\nabla^2F(w_t)\bigr).
$$

As a result, it is sufficient to (i) construct an en example where the largest eigenvalue of the matrix $\Gamma=\Pi_{t=1}^T D\Phi(w_t)$ is exponential in $T$.
(ii) uniformly bound the linear approximation errors.

In this sketch we focus on step (i) which is the core of the proof. 
For this step, we consider a feasible cylinder,
\[
    \mathcal W
    :=
    \left\{
    (x,z)\in\mathbb R^2\times\mathbb R:
    \|x\|_2\le r,\;
    -m\le z\le m
    \right\},
\]
for some $r,m > 0$, and parametrize the coordinates of the iterate as \(w=(x,z)\in\mathbb{R}^2\times\mathbb{R}\).
The reference trajectory evolves only in the \(z\)-direction: \(w_t=(0,z_t)\). Thus, \(\{z_t\}_{t=1}^T\) acts as a clock variable, while the transverse variable \(x\in\mathbb{R}^2\) captures directions orthogonal to the $z$-direction, where the instability is generated. 

To control the eigenvalues of $\Gamma$, we use a rotation matrix $P(z)$ to make the dynamics of the transverse plane rotate around the reference trajectory, such that, at step $t$ \[D\Phi(w_t)\approx
P(z_{t+1}) A P(z_t)^\top,
\] for some fixed matrix $A$. Telescoping and multiplying over \(T\) steps yields,
\[
\Gamma =
P(z_T) A^T P(z_0)^\top.
\]
Thus, up to outer rotations, the dynamics reduce to \(A^T\). Then,
choosing the objective $F$ and the regularization $R$ such that \(A\) has an eigenvalue \(\rho\approx 1+\eta\) , will cause an 
an exponential amplification of the perturbation. 
For an illustration of this see \cref{fig:endpoint_exponential_origin_mechanics}.

To achieve this, we construct the following \(z\)-dependent quadratic forms,
\[
    M(z)=P(z) B P(z)^\top,
    \quad
    H(z)=P(z) Q P(z)^\top,
\]
where \(B,Q\) are appropriately chosen symmetric matrices, and define
\[
    R(x,z)
    =
    \frac12 x^\top M(z)x + R_{\rm clock}(z),
    \quad
    F(x,z)
    \approx
     \frac12 x^\top H(z)x,
\]
where \(R_{\rm clock}\) is chosen to govern the motion of the reference trajectory in the \(z\)-direction.
Those $R$ and $F$ induce the required exponential matrix $\Gamma$ for 
$A \approx B^{-1} P(-\eta)(B- \eta Q)$. Thus, choosing the perturbation $p$ in the direction of the expanding eigenvector of $A$ induces the desired exponential expansion.

\qed

\section{Instability in Entropic Geometry}
\label{sec:KL}

We next study one of the fundamental examples of mirror descent: negative
entropy on the simplex, which underlies multiplicative weights and
KL-regularized updates. Unlike the regularizers considered in the previous section, negative entropy is ill-conditioned near the boundary of the simplex. Let
\(\simp=\{w\in\mathbb{R}_{++}^d:\sum_i w_i=1\}\) denote the
\(d\)-dimensional open simplex, and take
\(R(w)=\sum_i w_i\log w_i\). The corresponding Bregman divergence is the
Kullback--Leibler divergence,
\(D_R(w,w')=\sum_i w_i\log(w_i/w_i')\). 
For \(w\in\simp\), write \(w(i)\) for its \(i\)th coordinate and
\(\smash{w^{\min}}=\min_i w(i)\).

In this geometry, the instability mechanism becomes especially transparent.
Unlike the well-conditioned case studied in the preceding section, exponential
amplification already occurs for linear objectives. The mechanism is the nonuniform conditioning of negative entropy: a small primal perturbation in a low-mass coordinate can correspond to a large displacement in the dual variables, and in high dimension such low-mass coordinates are present for \emph{every} initialization. This local ill-conditioning is captured by the quantity
\(\initbeta\).

\begin{theorem}\label{thm:kl-lower}
Let $d\ge 2$, $\varepsilon\in(0,1/2]$, $\eta>0$, and $T\ge 1$. Then, for every
initialization point $\w_0\in\simp$, there exists a $1$-Lipschitz linear
objective \(F\) such that, if \(A\) is MD with negative entropy regularization
applied to \(F\) for \(T\) iterations with step size \(\eta\), then its
\(\ell_1\)-initialization stability satisfies
\[
    \delta_A(w_0, \varepsilon)  
    \geq 
    \tfrac{1}{5}
    \min\Bigl\{
    1,\,
    e^{\eta T} \varepsilon, \,
    \tfrac{ \varepsilon}{w_0^{\min}}
    \Bigr\} 
    =  
    \Omega \Bigl( 
    \min\Bigl\{
    1,\,
    e^{\eta T} \varepsilon, \,
    \epsilon\initbeta
    \Bigr\}
    \Bigr).
\]
In particular, if \(d\ge 1/\varepsilon\), then for any initialization \(w_0\in\simp\) it holds 
$
    \delta_A(w_0,\varepsilon)
    \ge
    \tfrac{1}{5}\min\{1, e^{\eta T} \varepsilon\}.
$
\end{theorem}
Specifically for the canonical stepsize $\eta = \Theta(\sqrt{\log d/T})$ the regime where MD achieves its $O(1/\sqrt{T})$ optimization guarantee with $d \geq 1/\epsilon$, the bound is still exponential:  
\[
    \delta_A(w_0, \varepsilon)  
    \geq 
    \tfrac{1}{5}
    \min\Bigl\{
    1,\,
    e^{\Theta(\sqrt{T\log d })} \varepsilon
    \Bigr\}.
\]

We remark that it was previously shown that \emph{a single step} of negative-entropy MD can be non-contractive, from a specific, near-boundary initialization~\citep{asi2021private}. 
In contrast, here
we show that such expansion can persist for \(T\) steps, yielding a lower bound
that grows exponentially with \(T\), and moreover, in the high-dimensional regime,
our bound holds uniformly over all initializations in the open simplex \(\simp\).

The proof uses the multiplicative-weights form of entropy MD. We perturb a lowest-mass coordinate of \(w_0\), where a small primal change creates the
largest logarithmic change in the dual variables of negative entropy. This produces two initializations that are close in \(\ell_1\) distance but far in the local dual geometry, with dual separation measured by \(\varepsilon\initbeta\). We then choose a linear objective that rewards this coordinate, so that the multiplicative update amplifies this initial dual separation over time.
Thus, the exponential amplification in the lower bound is driven by the local ill-conditioning of negative entropy in \(\ell_1\) geometry. In high dimension, every point in the simplex has a coordinate of mass at most \(1/d\), so this ill-conditioning is present for every initialization. More generally, \cref{lem:condition_l1} shows that such poor conditioning is unavoidable for regularizers compatible with the \(\ell_1\) geometry.

\begin{proof}[Proof of \cref{thm:kl-lower}]
Let $
i_m\in\arg\min_{i\in[d]}\w_0(i)$ and $\w_0(i_m)=w_0^{\min}$. Since $d\ge 2$, we have $w_0^{\min}\le 1/2$, and therefore $1-w_0^{\min}\ge \frac12> \frac{\varepsilon}{2}$.
Thus we may move mass $\varepsilon/2$ from the coordinates $i\neq i_m$ to coordinate $i_m$. For example, define $p\in\mathbb{R}^d$ by
\[
p(i_m)=\frac{\varepsilon}{2},
\qquad
p(i)= -\frac{\varepsilon}{2}\cdot \frac{\w_0(i)}{1-w_0^{\min}}
\quad\text{for } i\neq i_m.
\]
Then $\sum_{i=1}^d p(i)=0$, and since $\varepsilon/2< 1-w_0^{\min}$, we have $\w_0(i)+p(i)> 0$ for every
$i\neq i_m$. Therefore $\w_0':=\w_0+p\in\simp$.
Moreover,
\[
\|\w_0-\w_0'\|_1
=
\|p\|_1
=
\frac{\varepsilon}{2}
+
\sum_{i\neq i_m}\frac{\varepsilon}{2}\cdot\frac{\w_0(i)}{1-w_0^{\min}}
=
\varepsilon.
\]

Now define
\(
K:=
\min\left\{
e^{\eta T},
\frac{1}{w_0^{\min}},
\frac{1}{\varepsilon}
\right\}.
\)
Let $F$ be the linear function $F(\w)=\langle g,\w\rangle$, where
\[
g(i)=
\begin{cases}
-\dfrac{\log K}{\eta T}, & i=i_m,\\[1em]
0, & i\neq i_m.
\end{cases}
\]
Since $K\le e^{\eta T}$, we have $0\le \frac{\log K}{\eta T}\le 1$.
Hence $\|g\|_\infty\le 1$, so $F$ is $1$-Lipschitz with respect to
$\|\cdot\|_1$.

For negative entropy mirror descent with a fixed linear loss, the update has
the multiplicative-weights form
\[
\w_T(i)
=
\frac{\w_0(i)e^{-\eta T g(i)}}{
\sum_{j=1}^d \w_0(j)e^{-\eta T g(j)}
}.
\]
Thus coordinate $i_m$ is multiplied by $K$, while all other coordinates are
multiplied by $1$. Therefore, after $T$ steps 
\[
w_T(i_m)
=
\frac{Kw_0(i_m)}{1+(K-1)w_0(i_m)},
\qquad 
w_T'(i_m)
=
\frac{Kw_0'(i_m)}{1+(K-1)w_0'(i_m)}.
\]
Hence
\[
\|\w_T-\w_T'\|_1
\ge
2\left|w_T(i_m) - w+T'(i_m)\right| = 2\left| \frac{Kw_0(i_m)}{1+(K-1)w_0(i_m)} + \frac{Kw_0'(i_m)}{1+(K-1)w_0'(i_m)}\right|.
\]
Plugging in $w_0(i_m) = w_0^{\min}$ and $w_0'(i_m) = w_0^{\min} + \frac{\varepsilon}{2}$ gives by direct computation
\[
\|\w_T-\w_T'\|_1
\ge
\frac{K\varepsilon}{
\left(1+(K-1)(w_0^{\min}+\varepsilon/2)\right)
\left(1+(K-1)w_0^{\min}\right)
}.
\]

We now bound the denominator. Since $K\le 1/w_0^{\min}$,
\[
1+(K-1)w_0^{\min}
\le
1+Kw_0^{\min}
\le
2.
\]
Also, since $K\le 1/w_0^{\min}$ and $K\le 1/\varepsilon$,
\[
1+(K-1)\left(w_0^{\min}+\frac{\varepsilon}{2}\right)
\le
1+K w_0^{\min}+\frac{K\varepsilon}{2}
\le
1+1+\frac12
=
\frac52.
\]
Thus
\[
\|\w_T-\w_T'\|_1
\ge
\frac{K\varepsilon}{5} = \frac{1}{5}
\min\left\{
1,\,
\varepsilon e^{\eta T},\,
\frac{\varepsilon}{w_0^{\min}}
\right\}.
\]
as claimed. It remains to relate the term \(\varepsilon/w_0^{\min}\) to the local
smoothness parameter \(\initbeta\). For negative entropy,
\(
    \nabla R(w)(i)=1+\log w(i),
\)
and therefore
\[
    \varepsilon \initbeta
    =
    \sup_{\substack{p:\,\|p\|_1\le \varepsilon,\\ w_0+p\in\simp}}
    \left\|
        \log(w_0+p)-\log(w_0)
    \right\|_\infty .
\]
If \(\varepsilon\ge 2w_0^{\min}\), then the perturbation ball can approach the
boundary of the simplex, and hence \(\initbeta=\infty\). In this case,
\[
    \min\left\{1,\frac{\varepsilon}{w_0^{\min}}\right\}
    =
    1
    =
    \min\{1,\varepsilon\initbeta\},
\]
so the desired inequality holds. If \(\varepsilon<2w_0^{\min}\),
\[
    \varepsilon\initbeta
    =
    \log\left(
        \frac{w_0^{\min}}{w_0^{\min}-\varepsilon/2}
    \right) = \log\left(
        \frac{1}{1-\frac{\epsilon}{2w_0^{\min}}}
    \right).
\]
If \(\frac{\epsilon}{2w_0^{\min}}\ge 1/2\), then
\[
    \min\left\{1,\frac{\varepsilon}{w_0^{\min}}\right\}
    =
    1
    \ge
    \frac12 \min\{1,\varepsilon\initbeta\}.
\]
If \(\frac{\epsilon}{2w_0^{\min}}<1/2\), then using
\[
    \log\left(\frac{1}{1-x}\right)
    \le
    \frac{x}{1-x}
    \le
    2x, \quad 0\leq x \leq \frac12,
\]
we obtain
\[
    \frac12 \min\{1,\varepsilon\initbeta\}
    \le
    \frac12 \varepsilon\initbeta
    \le
    \frac{\epsilon}{2w_0^{\min}}
    \le
    \frac{\epsilon}{w_0^{\min}}
    =
    \min\left\{1,\frac{\varepsilon}{w_0^{\min}}\right\}.
\]
Combining the cases gives
\[
    \min\left\{1,\frac{\varepsilon}{w_0^{\min}}\right\}
    \ge
    \frac12
    \min\{1,\varepsilon\initbeta\},
\]
which concludes the proof.
\end{proof}
For completeness, we also state a matching exponential upper bound for the initialization instability of entropy-regularized mirror descent.
\begin{theorem}\label{thm:kl-upper-bound}
Let \(T\ge0\), \(\eta>0\), and \(\varepsilon>0\). For any \(G\)-Lipschitz and
\(L\)-smooth objective \(F\) with respect to the \(\ell_1\) norm, if \(A\) is
MD with negative entropy regularization applied to \(F\) for \(T\) steps with
step size \(\eta\), then for any initialization point \(w_0\in\simp\),
\[
    \delta_A(w_0,\varepsilon)
    \le
    \min\left\{2, e^{(2G+4L)\eta T} \varepsilon \right\}.
\]
\end{theorem}
The proof, deferred to \cref{sec:KL-upper-proof}, relies only on the
multiplicative-weights form of entropy MD and the Lipschitzness and smoothness
of \(F\), and in particular applies to non-convex objectives satisfying
these regularity assumptions.

We also give an analogous extension for MD with Legendre regularizers (see \cref{def:legendre} in 
\cref{sec:addition-upper}), a class that is central in online optimization and RL~\citep{cesa2006prediction}.

\section{Stabilizing Mirror Descent Through Anchoring}
\label{sec:stabilizing}
In this section, we discuss how to mitigate the instability of MD demonstrated in \cref{sec:euclidean,sec:KL}. Our approach is to modify MD by adding a Bregman regularization term centered at a reference point, denoted as \emph{the anchor}, rather than applying MD directly to the original objective. We propose two variants of such anchored methods and show that the resulting stability guarantees depend crucially on the choice of the reference point.

\subsection{Anchoring at the Initialization} 

We first consider the case where the anchor is the initialization itself, which, as discussed in the introduction, is closely related to common regularization schemes used in real-world applications. In particular, we study \emph{Initialization-Anchored MD}, described in \cref{alg:anchored-MD-init}.

\begin{algorithm}[t] 
\caption{Initialization-Anchored Mirror Descent}
\label{alg:anchored-MD-init}
\begin{algorithmic}
\State \textbf{Input:} \(L\)-smooth function \(F\), initialization \(\w_0\in\dom\), regularization parameter \(\mu\), no.\ of steps \(T\).
\For{\(t \gets 0\) \textbf{to} \(T-1\)}
    \[
    \w_{t+1} \gets 
    \arg\min_{\w \in \dom}
    \left\{
    \left\langle
        \nabla F(\w_t)
        +
        \mu\bigl(\nabla R(\w_t)-\nabla R(\w_0)\bigr),
        \w-\w_t
    \right\rangle
    +
    (\mu+L)D_R(\w,\w_t)
    \right\}.
    \]
\EndFor
\end{algorithmic}
\end{algorithm}

The following theorem shows that \cref{alg:anchored-MD-init} achieves low initialization stability while preserving the optimization guarantees of vanilla MD up to logarithmic factors. In particular, rather than exhibiting exponential dependence on \(\varepsilon\), the stability bound consists of two terms that are controlled in the well-conditioned case, in the well-conditioned case, where \(R\) is both strongly convex and
smooth: one measuring the sensitivity of the anchor to perturbations in the initialization, and another that decays with the number of iterations.

\begin{theorem}
\label{thm:stabilizing-alg-init}
Assume \(T\ge2\), \(L>0\), and \(F\) is convex, \(G\)-Lipschitz, and
\(L\)-smooth with respect to \(\|\cdot\|\) on a convex domain \(\dom\) of
diameter \(D\). Assume also that \(R\) is \(1\)-strongly convex with respect to
\(\|\cdot\|\). Let \(A\) be \cref{alg:anchored-MD-init} run with
\(
    \mu = {8L\log (T)}/{T}.
\)
Then, for any initialization \(w_0\in\dom\),
\[
    \delta_A(w_0,\varepsilon)
    \le \epsilon\initbeta  
    +
    \sqrt{\frac{GD}{LT\log T}}.
\]
If, in addition, \(R\) is \(\beta\)-smooth with respect to \(\|\cdot\|\), %
\[
    \delta_A(w_0,\varepsilon)
    \le
    \beta \epsilon
    +
    \sqrt{\frac{GD}{LT\log T}}.
\]
Moreover, if \(w^\star\in\arg\min_{w\in\dom}F(w)\), then the output
\(\w_T=A(w_0)\) satisfies
\[
    F(\w_T)-F(w^\star)
    \le
    \frac{8L D_R(w^\star,w_0) \log T}{T}
    +
    \frac{2GD}{T^2}.
\]
\end{theorem}

The proof, deferred to \cref{sec:stabilizing-proof}, builds on the relative smoothness framework of \citet{bauschke2017descent,lu2018relatively,attia2022uniform}. The update rule can be viewed as Mirror Descent applied to the anchored objective
\(
F(w)+\mu D_R(w,w_0),
\)
which is \(\mu\)-strongly convex relative to \(R\). Consequently, each trajectory contracts toward the minimizer of its corresponding anchored objective.
To compare two runs initialized at \(w_0\) and \(w_0+p\), we decompose their final distance into three terms: the distance of each trajectory to the minimizer of its own anchored objective, and the distance between the two anchored minimizers. The first and third terms decay with \(T\), while the second captures the discrepancy induced by the different anchors. This term is controlled by
\(
\varepsilon\initbeta.
\)
Therefore, when \(R\) is additionally \(\beta\)-smooth, the discrepancy is bounded by \(\beta\varepsilon\), yielding a meaningful stabilization guarantee.

\cref{thm:stabilizing-alg-init}
shows that anchoring MD at the initialization resolves the instability
problem in the well-conditioned case, where \(R\) is both strongly convex and
smooth. However, it does not resolve the ill-conditioned case. If \(R\) is not
smooth, the quantity
\(
    \epsilon\initbeta
\)
can be large even when \(\|p\|\) is small. %
The following theorem shows that this dependence on \(\varepsilon\initbeta\) is unavoidable for Initialization-Anchored MD. In particular, for negative entropy the method can remain exponentially unstable in precisely the ill-conditioned regimes where \(\varepsilon\initbeta\) is large. The proof is deferred to \cref{sec:ill-conditioned-anchor}.
\begin{theorem}[Lower bound for initialization-anchored entropy MD]
\label{thm:anchored-kl-lower}
Let \(d\ge 2\), \(\varepsilon\in(0,1/2]\), \(L>0\), \(\mu>0\), and
\(T\ge 1\). Let $R$ be the negative entropy on \(\simp\). Let \(A\) be \cref{alg:anchored-MD-init} run with this
regularizer over \(\simp\). Then, for every initialization
\(\w_0\in\simp\), there exists a \(1\)-Lipschitz linear
objective \(F\), which may depend on \(\w_0,\varepsilon,L,\mu,T\), such that,
\[
    \delta_A(\w_0,\varepsilon)
    \;\ge
    \frac{1}{10}
    \min\left\{
    1,\,
    \varepsilon
    \exp\left(
        (1-e^{-1})
        \min\left\{
        \frac{T}{L+\mu},
        \frac{1}{\mu}
        \right\}
    \right),
    \epsilon\initbeta
    \right\}.
\]
If additionally $d \geq 1/\epsilon$ then,
\[
    \delta_A(\w_0,\varepsilon)
    \ge
    \frac{1}{5}
    \min\left\{
    1,\,
    \varepsilon
    \exp\left(
        (1-e^{-1})
        \min\left\{
        \frac{T}{L+\mu},
        \frac{1}{\mu}
        \right\}
    \right)
    \right\}.
\]
\end{theorem}
Notably, the optimization-relevant regime is when \(\mu\) is small enough that the anchoring term does not dominate the original objective; in particular, for vanishing optimization error one typically takes $\mu$ which decreases with $T$. For example, when $\mu = 1/\sqrt{T}$, the bound become \(\Omega(1,\, \varepsilon e^{\Omega(\sqrt{T})})\).

\subsection{Anchoring at a Fixed Point}

The theorem above shows that anchoring at the initialization does not preclude exponentially-increasing initialization stability in the ill-conditioned initialization case. To address this, we introduce a second variant in which the anchor is fixed independently of the initialization.

\begin{algorithm}[t] 
\caption{Fixed-Anchor Mirror Descent}
\label{alg:anchored-MD-fixed}
\begin{algorithmic}
\State \textbf{Input:} Anchor \(w_a\in\dom\), \(L\)-smooth function \(F\), initialization \(\w_0\in\dom\), regularization parameter \(\mu\), number of steps \(T\).
\For{\(t \gets 0\) \textbf{to} \(T-1\)}
    \[
    \w_{t+1} \gets 
    \arg\min_{\w \in \dom}
    \left\{
    \left\langle
        \nabla F(\w_t)
        +
        \mu\bigl(\nabla R(\w_t)-\nabla R(w_a)\bigr),
        \w-\w_t
    \right\rangle
    +
    (\mu+L)D_R(\w,\w_t)
    \right\}.
    \]
\EndFor
\end{algorithmic}
\end{algorithm}

In particular, we study \emph{Fixed-Anchor MD}, described in \cref{alg:anchored-MD-fixed}, where the additional Bregman regularization term is centered at a fixed reference point \(w_a\). 
The following theorem shows that, in contrast to \cref{alg:anchored-MD-init}, 
\cref{alg:anchored-MD-fixed} is stable even in the ill-conditioned setting.

\begin{theorem}
\label{thm:stabilizing-alg-fixed}
Assume \(T\ge2\), \(L>0\), and \(F\) is convex, \(G\)-Lipschitz, and
\(L\)-smooth with respect to \(\|\cdot\|\) on a convex domain \(\dom\). Assume also that \(R\) is \(1\)-strongly convex with respect to
\(\|\cdot\|\). Let \(A\) be \cref{alg:anchored-MD-fixed} run with
\(
    \mu = {8L\log (T)}/{T}.
\)
Then, for any initialization \(w_0\in\dom\),
\[
    \delta_A(w_0,\varepsilon)
    \le
    \frac{2}{T}\sqrt{
        D_R(w^\star_\mu,w_0)
        +
        D_R(w^\star_\mu,w_0')
    }
    ,
\]
where \(w_0'\) denotes the perturbed initialization and
\(
    w^\star_\mu \in \arg\min_{w\in\dom}
    \left\{
        F(w)+\mu D_R(w,w_a)
    \right\}.
\)
Moreover, if \(w^\star\in\arg\min_{w\in\dom}F(w)\), then the output
\(\w_T=A(w_0)\) satisfies
\[
    F(\w_T)-F(w^\star)
    \le
    \frac{8L D_R(\w^\star_{\mu},\w_0) \log T}
    {T(T^2 - 1)}
    +
    \frac{8L  D_R(\w^\star,w_a) \log T}{T}.
\]
\end{theorem}
The proof, deferred to \cref{sec:stabilizing-proof}, follows a similar approach to that of \cref{thm:stabilizing-alg-init} and relies on the fact that adding the regularization term \(\mu D_R(w,w_a)\) makes the objective \(\mu\)-strongly convex relative to \(R\).

The key difference is that the anchor remains fixed under perturbations of the initialization. Consequently, two nearby initializations are regularized toward the same reference point and optimize the same anchored objective. As a result, both trajectories contract toward the same minimizer, rather than toward different minimizers as in the initialization-anchored setting. This allows us to leverage the exponential convergence of the regularized dynamics to obtain the desired stability guarantee.

Notably, the analysis shows that stability emerges once both trajectories become sufficiently close to the minimizer, and therefore the resulting bound does not depend explicitly on the perturbation size \(\varepsilon\). An interesting open question is whether one can obtain stability guarantees with explicit dependence on \(\varepsilon\), showing that \(\varepsilon\)-close initializations generate trajectories that remain close throughout the optimization process.

\section{Discussion and Limitations}
\label{sec:disc}

In this work, we study the initialization stability of mirror descent under
general norms and regularizers. Focusing on convex and \(L\)-smooth objectives,
we prove exponential lower bounds in two settings: well-conditioned
nonquadratic regularizers in Euclidean geometry, and the canonical KL
regularizer on the simplex. These results establish an exponential separation
between the stability of GD, or more generally quadratic MD, and the broader MD
paradigm. Notably, our constructions are convex, so the instability already
arises in the classical setting for which MD was originally developed. Somewhat 
surprisingly, for KL regularization, the phenomenon appears even for fixed linear 
objectives and even when MD is initialized near the center of the simplex.

\paragraph{Open questions.}
Our primary focus here is initialization stability, and a natural next step is to
investigate its connection to algorithmic stability and generalization. In
particular, it would be interesting to understand whether sensitivity to
initialization translates into sensitivity to perturbations of the training
sample, potentially linking the dynamical phenomena studied here with
generalization guarantees for (S)GD in non-smooth settings
\citep{amir2021sgd,schliserman22a,livni2024sample,vansover2025rapid}. This
question is especially intriguing because initialization stability concerns
perturbations of the primal iterates, whereas algorithmic stability is typically
driven by perturbations that enter through gradients, or dual variables. These
notions coincide in the quadratic case, but they may differ substantially for
general mirror descent dynamics.

Our findings suggest several additional directions for future research:
\begin{itemize}[leftmargin=!]
    \item 
    The lower bounds established in this work are stated for worst-case
    initialization perturbations. A natural direction is to study average-case
    notions of initialization stability, for example by considering
    perturbations sampled uniformly at random from an \(\varepsilon\)-ball and
    determining whether exponential instability still occurs.

    \item 
    Another interesting question is whether the lower bound of
    \cref{thm:kl-lower} extends to other widely used simplex regularizers, such
    as the log-barrier and Tsallis regularization. This is motivated in part by
    the recent work of \citet{schlisselberg2025hidden} in online linear
    optimization, which studies the sensitivity of Mirror Descent to
    approximation errors from inexact updates and shows a qualitative
    separation between KL regularization and other regularizers. It would be
    interesting to determine whether a similar separation appears for
    initialization stability.

    \item 
    It would also be interesting to study the stability effects of
    clipping-based regularization. While many post-training methods for LLMs use
    KL regularization within MD-like updates, recent works increasingly rely on
    implicit regularization induced by clipping
    \citep{yu2025dapo,rastogi2025magistral,khatri2025art,chen2025minimax}.
    Understanding whether such methods exhibit similar instability phenomena is
    an important direction for future work.
\end{itemize}

\paragraph{Limitations.}
Our results are worst-case lower bounds, and the initialization perturbations
are chosen adversarially. This perspective is useful because it rules out
general initialization-stability upper bounds for MD under the assumptions
considered here. At the same time, understanding which additional assumptions
capture the empirical success of MD-like methods remains an important direction
for future work.

\section*{Acknowledgments}

This project has received funding from the European Research Council (ERC) under the European
Union's Horizon 2020 research and innovation program (grant agreement No.\ 101078075). Views and
opinions expressed are however those of the author(s) only and do not necessarily reflect those of the
European Union or the European Research Council. Neither the European Union nor the granting
authority can be held responsible for them. This work received additional support from the Israel
Science Foundation (ISF; grant numbers 2549/19 and 3174/23), from the Council for Higher Education in Israel under a Moonshot Project, and a fellowship from the Tel Aviv University Center for AI and Data Science (TAD).
SVH is partially supported by the TAD Excellence Program for Doctoral Students in Artificial Intelligence and Data Science of the Tel Aviv University Center for AI and Data Science (TAD).
OS is supported by the European Research Council (ERC) under the European Union’s Horizon 2020 research and innovation program (grant agreement No. 882396), by the Israel Science Foundation, by a grant from the Tel Aviv University Center for AI and Data Science (TAD), by the TAD Excellence Program for Doctoral Students in Artificial Intelligence and Data Science from the Tel Aviv University Center for AI and Data Science (TAD) and from the Israeli Council for Higher Education (CHE) Fellowship for Outstanding PhD Students in Data Science.

\bibliographystyle{plainnat}
\bibliography{ref}

\newpage
\appendix
\crefalias{section}{appendix}
\section{Additional Upper Bounds}\label{sec:addition-upper}
\subsection{Quadratic Regularizer}
In this section we assume the regularizer is a quadratic function:
\begin{align*}
    R(w) = \frac12 w^\top Aw + \langle b,w\rangle
\end{align*}

The primal norm is then $\norm{\cdot}_A$.

\begin{lemma}\label{lem:remove_grad}
Let $A,B$ be symmetric matrices such that $0\preceq B\preceq 2A$. Then, for every $w\in\mathbb{R}^d$,
\begin{align*}
    \norm{(A - B)w}_{A^{-1}} \le \norm{Aw}_{A^{-1}}
\end{align*}
\end{lemma}
\begin{proof}
Since $0\preceq B\preceq 2A$, we have
$A^{-1/2}BA^{-1/2}\preceq 2I$, and therefore
$BA^{-1}B\preceq 2B$. Thus
\begin{align*}
    (A-B)A^{-1}(A-B) &= A - 2B + BA^{-1}B\\
    &\preceq A-2B + 2B\\
    &= A
\end{align*}

Which means:
\begin{align*}
    \norm{(A - B)w}_{A^{-1}}^2 &= w^T(A-B)A^{-1}(A-B)w\\
    &\le w^TAw\\
    &= w^TAA^{-1}Aw\\
    &= \norm{Aw}_{A^{-1}}^2
\end{align*}
\end{proof}

\begin{lemma}
For every $w,w'$:
\begin{align*}
    \inprod{\nabla R(w) - \nabla R(w'),\, w - w'} = \norm{w - w'}_A^2
\end{align*}
\end{lemma}
\begin{proof}
\begin{align*}
    \inprod{\nabla R(w) - \nabla R(w'),\, w - w'} &= \inprod{A(w - w'),\, w - w'} = \norm{w - w'}_A^2
\end{align*}
\end{proof}
First we prove the stability w.r.t. the Mahalanobis norm $\|\cdot \|_A$ in the following Theorem. 
\begin{theorem}\label{lem:contractiveness_quadratic}
Assume $F$ is $L$ smooth w.r.t $\norm{\cdot}_A$ and $\eta \le \frac{2}{L}$.
For every $t$:
\begin{align*}
    \norm{w_{t+1} - w_{t+1}'}_A \le \norm{w_t - w_t'}_A.
\end{align*}
\end{theorem}
\begin{proof}
From first order optimality conditions:
\begin{align*}
    \inprod{\eta \nabla F(w_t) + \nabla R(w_{t+1}) - \nabla R(w_t), w_{t+1}' -  w_{t+1}} &\ge 0 \\
    \inprod{\eta \nabla F(w_t') + \nabla R(w_{t+1}') - \nabla R(w_t'), w_{t+1} -  w_{t+1}'} &\ge 0
\end{align*}

Summarizing them:
\begin{align*}
    & \inprod{\nabla R(w_{t}) - \eta \nabla F(w_t) - \nabla R(w_{t}') + \eta \nabla F(w_t'),\, w_{t+1} - w_{t+1}'} \ge\\
    & \inprod{\nabla R(w_{t+1}) - \nabla R(w_{t+1}'), w_{t+1} - w_{t+1}'} = \norm{w_{t+1} - w_{t+1}'}_A^2
\end{align*}

Using Holder:
\begin{align*}
    \norm{w_{t+1} - w_{t+1}'}_A^2 &\le \norm{\nabla R(w_{t}) - \eta \nabla F(w_t) - \nabla R(w_{t}') + \eta \nabla F(w'_t)}_{A^{-1}}\norm{w_{t+1} - w_{t+1}'}_A\\
    \implies \norm{w_{t+1} - w_{t+1}'}_A &\le \norm{\nabla R(w_{t}) - \eta \nabla F(w_t) - \nabla R(w_{t}') + \eta \nabla F(w'_t)}_{A^{-1}}
\end{align*}

Let
\begin{align*}
    B_t:=\int_0^1 \nabla^2 F\!\left(w_t' + s(w_t-w_t')\right)\,ds .
\end{align*}
Then $\nabla F(w_t)-\nabla F(w_t')=B_t(w_t-w_t')$.

Notice that since $F$ is convex and $L$-smooth w.r.t $\norm{\cdot}_A$, $0\preceq B_t \preceq LA$. Since $\eta \le \frac{2}{L}$, we have $0\preceq \eta B_t\preceq 2A$. Thus:
\begin{align*}
     \norm{w_{t+1} - w_{t+1}'}_A &\le \norm{(A - \eta B_t)(w_t - w_t')}_{A^{-1}} \le \norm{A(w_t - w_t')}_{A^{-1}} = \norm{w_t - w_t'}_A
\end{align*}
The last inequality is due to \Cref{lem:remove_grad}.
\end{proof}
Finally we will extend the previous Theorem to a general norm:
\begin{theorem}\label{thm:quad-upper}
Assume $F$ is $L$-smooth w.r.t some $\norm{\cdot}$, $R$ is $\alpha$-strongly convex and $\beta$-smooth w.r.t $\norm{\cdot}$ and $\eta \le \frac{2\alpha}{L}$.
For every $t$:
\begin{align*}
    \norm{w_t - w_t'} \le \frac{\beta}{\alpha}\norm{w_0 - w_0'}
\end{align*}
\end{theorem}
\begin{proof}[Proof of \cref{thm:quad-upper}]
Notice that $F$ is $L/\alpha$-smooth w.r.t $\norm{\cdot}_A$. Thus, from \Cref{lem:contractiveness_quadratic}:
\begin{align*}
    \norm{w_t - w_t'}_A &\le \norm{w_0 - w_0'}_A\\
    \implies \alpha \norm{w_t - w_t'} &\le \beta\norm{w_0 - w_0'},
\end{align*}
which concludes the proof.
\end{proof}

\subsection{Legendre Regularizers}

In this section we assume the decision set $\dom$ is the intersection between some convex set $\mathcal{K}$ and a linear equality constraint $\mathcal C= \{ w \in \mathbb{R}^d \colon \; Aw = b\}$ for some $A,b$. Additionally, we assume that the regularizer is a Legendre similarly to the definition in \citep[Chapter 11.2]{cesa2006prediction}.
\begin{definition}[Legendre function]\label{def:legendre}
We call \emph{Legendre} any function $F : \mathcal{K} \to \mathbb{R}$ such that

\begin{enumerate}
    \item $\mathcal{K} \subseteq \mathbb{R}^d$ is nonempty and its interior
    $\operatorname{int}(\mathcal{K})$ is convex;

    \item $F$ is strictly convex with continuous first partial derivatives
    throughout $\operatorname{int}(\mathcal{K})$;

    \item if $\mathbf{x}_1, \mathbf{x}_2, \ldots \in \mathcal{K}$ is a sequence
    converging to a boundary point of $\mathcal{K}$, then
    \[
    \|\nabla F(\mathbf{x}_n)\| \to \infty
    \qquad \text{as } n \to \infty .
    \]
\end{enumerate}
\end{definition}
Intuitively this definition makes sure that the minimizer of the MD objective will always be inside $\mathcal{K}$, and we only need to project to the linear constraint. The MD step is thus:
\begin{align*}
    w_{t+1} = \arg \min_{w \in \mathcal{C} \cap \mathcal{K}}\left\{ \eta \inprod{w,\, \nabla F(w_t)} + D_R(w, w_t)\right\}.
\end{align*}

This is a natural extension of the simplex with many popular regularizers (e.g., negative entropy, log barrier and Tsallis entropy) with $\mathcal{K} = \{ w\in  \mathbb{R}^d |\; w(i) \ge 0 \quad \forall i\}$ and $\mathcal C = \{w \in \mathbb{R}| \|w\|_1=1\}$.

We also denote $Ker(A) = \{w \in \mathbb{R}^d | Aw = 0\}$. We now state and prove a few useful Lemmas before proving an upper bound on the stability of MD in this setting.

\begin{lemma}\label{lem:legendre_int}
For every Legendre function $\Phi\colon \; \mathcal{K} \to \mathbb{R}$, every minimizer of \(\Phi\) over \(\mathcal K\) lies in
\(\operatorname{int}\mathcal K\).
\end{lemma}

\begin{proof}
Since \(\mathcal K\) is compact and \(\Phi\) is continuous,
\(\Phi\) attains its minimum on \(\mathcal K\). Let \(w^\star\in\mathcal K\)
be a minimizer. We show that \(w^\star\notin\partial\mathcal K\).

Assume toward a contradiction that \(w^\star\in\partial\mathcal K\). 

Choose \(u\in\operatorname{int}\mathcal K\), and define
\[
    w_s := (1-s)w^\star+s u,
    \qquad s\in(0,1].
\]
Then \(w_s\in\operatorname{int}\mathcal K\) for every \(s\in(0,1]\), and
\(w_s\to w^\star\) as \(s\downarrow0\).

Since \(u\in\operatorname{int}\mathcal K\), there exists \(\rho>0\) such that
\[
    u+\rho v\in\operatorname{int}\mathcal K
    \qquad\text{for every } \|v\|\le 1 .
\]
Because \(\Phi\) is convex and finite on \(\operatorname{int}\mathcal K\), it is
locally bounded on compact subsets of \(\operatorname{int}\mathcal K\). Thus
\[
    M:=\sup_{\|v\|\le1}\Phi(u+\rho v)-\Phi(w^\star)<\infty .
\]

    By convexity, for every \(z\in\mathcal K\),
\[
    \Phi(z)\ge
    \Phi(w_s)+\langle \nabla\Phi(w_s),z-w_s\rangle .
\]
Taking \(z=u+\rho v\), where \(\|v\|\le1\), and using that \(w^\star\) is a
minimizer, we get
\[
    \langle \nabla\Phi(w_s),u+\rho v-w_s\rangle
    \le
    \Phi(u+\rho v)-\Phi(w_s)
    \le
    \Phi(u+\rho v)-\Phi(w^\star)
    \le M .
\]
Since
\[
    u+\rho v-w_s=(1-s)(u-w^\star)+\rho v,
\]
this gives
\[
    (1-s)\langle \nabla\Phi(w_s),u-w^\star\rangle
    +
    \rho\langle \nabla\Phi(w_s),v\rangle
    \le M .
\]

We now show that the first term is nonnegative. Applying the same convexity
inequality with \(z=w^\star\), we obtain
\[
    \Phi(w^\star)
    \ge
    \Phi(w_s)+\langle \nabla\Phi(w_s),w^\star-w_s\rangle .
\]
Since \(w^\star\) is a minimizer, \(\Phi(w_s)\ge\Phi(w^\star)\). Hence
\[
    \langle \nabla\Phi(w_s),w_s-w^\star\rangle\ge0.
\]
Because \(w_s-w^\star=s(u-w^\star)\), this implies
\[
    \langle \nabla\Phi(w_s),u-w^\star\rangle\ge0.
\]
Therefore, for every \(\|v\|\le1\),
\[
    \rho\langle \nabla\Phi(w_s),v\rangle\le M .
\]
Applying the same inequality to \(-v\), we get
\[
    |\langle \nabla\Phi(w_s),v\rangle|
    \le
    \frac{M}{\rho}
    \qquad\text{for every }\|v\|\le1.
\]
Thus
\[
    \|\nabla\Phi(w_s)\|_*\le \frac{M}{\rho}
    \qquad\text{for every }s\in(0,1].
\]
This contradicts the assumed boundary behavior, since
\(w_s\in\operatorname{int}\mathcal K\) and \(w_s\to w^\star\in\partial\mathcal K\).
Therefore \(w^\star\notin\partial\mathcal K\), and every minimizer lies in
\(\operatorname{int}\mathcal K\).
\end{proof}

\begin{lemma}\label{lem:strong_convexity}
For every strongly convex $R$:
\begin{align*}
    \norm{x-y}^2 \le \inprod{\nabla R(x) - \nabla R(y), x-y}
\end{align*}
\end{lemma}
\begin{proof}
\begin{align*}
    \inprod{\nabla R(x) - \nabla R(y), x-y} &= \inprod{\nabla R(x), x-y} + \inprod{\nabla R(y), y-x}\\
    &= D_R(x,y) + D_R(y,x)\\
    &\ge \frac{1}{2}\norm{x - y}^2 + \frac{1}{2}\norm{x - y}^2 \tag{strong convexity}\\
    &= \norm{x - y}^2
\end{align*}
\end{proof}

\begin{lemma}\label{lem:seq}
Consider the given sequence, for some $b,c$:
\begin{align*}
    a_t = b\sum_{s=1}^{t-1} a_s + c
\end{align*}
Then:
\begin{align*}
    a_t = c(1+b)^{t-1} \leq ce^{b(t-1)}
\end{align*}
\end{lemma}
\begin{proof}
We'll prove by induction. For $t=1$:
\begin{align*}
    a_1 = b\sum_{s=1}^{0} a_s + c = c
\end{align*}

Assume true for $t$, we have:
\begin{align*}
    a_{t+1} &= b\sum_{s=1}^{t} a_s + c\\
    &= ba_{t} + b\sum_{s=1}^{t-1} a_s + c\\
    &= ba_{t} + a_{t}\\
    &= (1+b)a_{t}
\end{align*}
and the claim follows.
\end{proof}

\begin{lemma}\label{lem:ker_A}
For every Legnedre function $\Phi\colon\; \mathcal{K}\to \mathbb{R}$ , let $w^*$ be its minimizer in $\dom$. Then, for every $v\in Ker(A)$:
\begin{align*}
    \inprod{\nabla \Phi(w^*), v} = 0.
\end{align*}
\end{lemma}
\begin{proof}
Fix $v\in Ker(A)$. 
Since $\Phi$ is Legendre w.r.t $\mathcal{K}$, $w^* \in \operatorname{int}(\mathcal{K})$ (\Cref{lem:legendre_int}), which means that there is $\epsilon > 0$ such that $w + \epsilon v \in \mathcal{K}$. Additionally,
since $A(w^* + \epsilon v) = Aw^* + A\epsilon v = b + 0 = b$, $w^* + \epsilon v$ is also in the $\mathcal{C}$. Hence, $w^* + \epsilon v \in \dom$.

From first order optimality conditions:
\begin{align*}
    \inprod{\nabla \Phi(w^*), w^* + \epsilon v- w^*} &\ge 0,\\
    \implies \inprod{\nabla \Phi(w^*), v} &\ge 0.
\end{align*}
We can do the same thing for $w^* - \epsilon v$, which means:
\begin{align*}
    \inprod{\nabla \Phi(w^*), -v} &\ge 0,
\end{align*}
which concludes the proof.
\end{proof}

\begin{theorem}\label{thm:legendre-upper} Suppose $F$ is $L$-smooth w.r.t. $\| \cdot \|$ and $R$ is $1$-strongly convex w.r.t. $\| \cdot \|$ and also the $R$ Legendre w.r.t. some convex set $\mathcal{K}$. Also the the feasible set $\dom = \mathcal{K} \cap \mathcal{C}$ for some $\mathcal{C} = \{w\in \mathbb{R}^d | Aw=b\}$. Denote by $A$ the MD algorithm run with regularizer $R$ on $F$ for $T$ steps with step size $\eta$ then for any $w_0 \in \dom$ it holds that:
\begin{align*}
    \delta_A(w_0, \varepsilon)  \le \max_{\|p\|\leq \varepsilon}\frac{1}{\alpha}\norm{\nabla R(w_0) - \nabla R(w_0+p)}\cdot e^{\eta T  L/\alpha}
\end{align*}
\end{theorem}
\begin{proof}
For $\Phi = \eta F + R$, $\Phi$ has Legendre w.r.t. $\mathcal{K}$. Fix some $w_0,w_0'$ and fix some $v\in Ker(A)$. From \Cref{lem:ker_A}:
\begin{align*}
    \inprod{\eta \nabla F(w_t) + \nabla R(w_{t+1}) - \nabla R(w_t), v} &= 0 \\
    \inprod{\eta \nabla F(w_t') + \nabla R(w_{t+1}') - \nabla R(w_t'), v} &= 0
\end{align*}

Subtracting them:
\begin{align*}
    \inprod{\nabla R(w_{t+1}) - \nabla R(w_{t+1}'), v} &= \inprod{\eta \nabla F(w_t') - \eta \nabla F(w_t) , v} + \inprod{\nabla R(w_{t}) - \nabla R(w_{t}'), v}  \\
    &\le \inprod{\nabla R(w_{t}) - \nabla R(w_{t}'), v} + \eta L\norm{w_t - w_t'}\norm{v}
\end{align*}
The inequality is from the smoothness of $F$.

Thus, for every $t$:
\begin{align*}
    \inprod{\nabla R(w_{t}) - \nabla R(w_{t}'), v} \le \inprod{\nabla R(w_{0}) - \nabla R(w_{0}'), v} + \sum_{s=1}^{t}\eta L\norm{w_s - w_s'}\norm{v}
\end{align*}

For $v=w_t - w_t'$:
\begin{align*}
    \inprod{\nabla R(w_{t}) - \nabla R(w_{t}'), w_t - w_t'} &\le \inprod{\nabla R(w_{0}) - \nabla R(w_{0}'), w_t - w_t'} + \sum_{s=1}^{t}\eta L\norm{w_s - w_s'}\norm{w_t - w_t'}\\
    \implies \alpha\norm{w_t - w_t'}^2 &\le \dual{\nabla R(w_0) - \nabla R(w_0)}\norm{w_t - w_t'} + \sum_{s=1}^{t}\eta L\norm{w_s - w_s'}\norm{w_t - w_t'}\\
    \implies \alpha\norm{w_t - w_t'} &\le \dual{\nabla R(w_{0}) - \nabla R(w_{0}')} + \sum_{s=1}^{t}\eta L\norm{w_s - w_s'}
\end{align*}
The first is due to \Cref{lem:strong_convexity} and Holder inequality, the second is a division by $\norm{w_t - w_t'}$, and the third is from the smoothness of $R$. \Cref{lem:seq} concludes the proof.
\end{proof}

\section{Ill-Conditioning in \texorpdfstring{$\ell_1$}{l1} Geometry}\label{sec:ill-conditioning}
We begin by showing that the negative entropy is locally ill-conditioned at every point in the simplex, as formalized in the following lemma.
\begin{lemma} \label{lem:local-condition-kl}
Let $R(w)=\sum_{i=1}^d w_i\log w_i$ and let
$w\in \simp$. Define
\[
\beta(w)=\sup_{h\in T\setminus\{0\}}
\frac{h^\top \nabla^2R(w)h}{\|h\|_1^2},
\qquad
\alpha(w)=\inf_{h\in T\setminus\{0\}}
\frac{h^\top \nabla^2R(w)h}{\|h\|_1^2},
\]
where $T=\{h\in\mathbb{R}^d:\sum_i h_i=0\}$, and set
$\kappa(w)=\beta(w)/\alpha(w)$. Then for any $w \in \simp$
\[
\beta(w)\ge \frac{2d-1}{4}, 
\qquad 
\kappa(w) = \Omega(d) .
\]
\end{lemma}
\begin{proof}
Since $\nabla^2 R(w)=\operatorname{diag}(1/w_1,\dots,1/w_d)$, choosing $h$
supported on the two smallest coordinates with values $1/2$ and $-1/2$
gives
\[
\beta(w)\ge \frac14\left(\frac1a+\frac1b\right).
\]
Since $a\le 1/d$ and $b\le 1/(d-1)$, this also gives
\[
\beta(w)\ge \frac{2d-1}{4}.
\]

It remains to lower bound the condition number. For any nonempty proper
subset $S\subset[d]$, write $p=\sum_{i\in S}w_i$. Define
$h_i=\frac12 w_i/p$ for $i\in S$ and
$h_i=-\frac12 w_i/(1-p)$ for $i\notin S$. Then $h\in T$ and
$\|h\|_1=1$, hence
\[
\alpha(w)
\le
h^\top \nabla^2 R(w)h
=
\frac{1}{4p}+\frac{1}{4(1-p)}
=
\frac{1}{4p(1-p)}.
\]
Therefore
\[
\kappa(w)
=
\frac{\beta(w)}{\alpha(w)}
\ge
p(1-p)\left(\frac1a+\frac1b\right).
\]

We now choose $S$. First suppose that $w^{\max}\le 1/2$. Then there exists
a nonempty proper subset $S$ with $p\in[1/4,3/4]$: if some coordinate has
mass at least $1/4$, take that coordinate; otherwise add coordinates until
the partial sum first exceeds $1/4$, which gives a sum at most $1/2$.
Thus $p(1-p)\ge 3/16$. Since
$\frac1a+\frac1b\ge 2d-1$, we obtain
\[
\kappa(w)\ge \frac{3}{16}(2d-1)\ge \frac{3}{16}d .
\]

Now suppose that $w^{\max}>1/2$. Let $S$ be the singleton containing the
maximal coordinate. Then
$p(1-p)=w^{\max}(1-w^{\max})$. If $d\ge 3$, the two smallest coordinates lie outside
$S$, so $a\le (1-w^{\max})/(d-1)$ and $b\le (1-w^{\max})/(d-2)$. Therefore
\[
\frac1a+\frac1b
\ge
\frac{d-1}{1-w^{\max}}+\frac{d-2}{1-w^{\max}}
=
\frac{2d-3}{1-w^{\max}}.
\]
Thus
\[
\kappa(w)
\ge
w^{\max}(1-w^{\max})\cdot \frac{2d-3}{1-w^{\max}}
=
w^{\max}(2d-3)
\ge
\frac{2d-3}{2}.
\]
For $d\ge 3$, this is at least $d/6$. The case $d=2$ is immediate, since
the tangent space is one-dimensional and hence $\kappa(w)=1=d/2$.
\end{proof}
We then show that this phenomenon is not specific to negative entropy, but is in fact unavoidable for any regularizer compatible with \(\ell_1\) geometry. This is formalized in the following lemma.
\begin{lemma}[Condition number lower bound for the $\ell_1$ norm]\label{lem:condition_l1}
    Suppose $f$ is convex, $\beta$-smooth w.r.t. the $\beta$ norm and $\alpha$-strongly-convex w.r.t. the $\ell_1$ norm. Suppose also that $f$ is twice differentiable. Then $\frac{\beta}{\alpha} \geq d$. 
\end{lemma}
\begin{proof}
Fix some $x$, we have for every $y$
$$\frac{\alpha}{2}\|x-y\|_1^2 \leq D_f(y,x) \leq \frac{\beta}{2} \|x-y\|_1^2$$
Fix some $h \in \mathbb{R}^d$ and choose $y = x+th$ then,
$$\frac{\alpha}{2}t^2\|h\|_1^2 \leq D_f(y,x) \leq \frac{\beta}{2} t^2\|h\|_1^2$$
and from second order expansion 
\begin{align*}
    D_f(y,x) & = f(y) - f(x) - \langle \nabla f(x), y-x \rangle \\
    & = \frac{1}{2}(y-x)^T \nabla^2f(x) (y-x) +o((y-x)) \\
    & = \frac{1}{2}t^2 h^T \nabla^2f(x) h + o(t^2\|h\|_2^2) \\
    & = \frac{1}{2}t^2 h^T \nabla^2f(x) h + o(t^2) \tag{$h$ is fixed}
\end{align*}
overall
$$\frac{\alpha}{2}t^2\|h\|_1^2 \leq \frac{1}{2}t^2 h^T \nabla^2f(x) h + o(t^2) \leq \frac{\beta}{2} t^2\|h\|_1^2$$ Now divide the two sides by $\frac{t^2}{2}$ and take $t \rightarrow 0$ and we get:
$$\alpha\|h\|_1^2 \leq h^T \nabla^2f(x) h \leq \beta \|h\|_1^2 \quad \forall h$$
Since $f$ is convex there exists some $H$ such that $\nabla^2f(x) = X^TX$ and so:
$$\alpha\|h\|_1^2 \leq \|Xh\|_2^2 \leq \beta \|h\|_1^2 \quad \forall h$$
Note that choosing $h = e_i$ for the $i$-th basis vector gives:
$$\alpha \leq \|Xe_i\|_2^2 \leq \beta$$
Now for a uniformly random $h \in \{\pm 1\}^d$ we have:
$$\mathbb{E}_{h\in\{\pm 1\}^d} [\|Xh\|_2^2] =  \sum_{i=1}^d \|Xe_i\|_2^2 \leq d \beta$$
So there exists some $h_0 \in \{\pm 1\}^d $ such that:
$$\|Xh_0\|_2^2 \leq d\beta$$
but on the other hand $\|h_0\|_1^2 = d^2 $ so from the lower bound we get
$$d^2\alpha = \alpha \|h_0\|_1^2 \leq \|Xh_0\|_2^2 \leq d\beta$$
overall:
$$\frac{\beta}{\alpha} \geq d$$
  
\end{proof}
Finally, we give a simple example illustrating the ill-conditioning of the negative entropy, establishing an exponential lower bound already in dimension \(1\) using a linear function. Since all norms are equivalent in \(d=1\), this extends more generally.
\begin{lemma} \label{thm:1d-kl-lower}
Let $d=1$. Let $A$ denote MD with the negative entropy run for with $T$ steps and step size $0<\eta\leq 1$ with $\dom = [0,1]$. Then for any $\varepsilon\in (0,\frac{1}{2})$, and any $w_0\in(0,\frac{1}{2})$, we have for $T \leq \frac{1}{\eta}\log (1/(w_0+\epsilon))$
$$\delta_A(w_0, \varepsilon) \geq \varepsilon e^{\eta T/2}.$$
\end{lemma}
\begin{proof}
Set $F(w):=-(\ln(1+\eta)/\eta)\, w$.
Since $\ln(1+\eta)/\eta \le 1$, the linear objective $F$ is $1$-Lipschitz. Fix some $w_0 \in (0,1/2)$ and let $w_0' = w_0 + \varepsilon$, since $w_0,\varepsilon\leq \frac{1}{2}$ we have that $w_0' \in \dom$. Since $F$ is linear, the MD updates are,
\[
\nabla R(w_{t+1}) = \nabla R(w_t) - \eta F' = \nabla R(w_t) + \ln(1+\eta),
\]
and similarly for $w_t'$. Since $\nabla R(w) = 1+\ln w$, we obtain $\ln w_{t+1} = \ln w_t + \ln(1+\eta)$, which implies $w_{t+1} = (1+\eta)w_t$.
The same computation gives $w_{t+1}' = (1+\eta)w_t'$. Therefore
\[
w_t = (1+\eta)^t w_0,
\qquad
w_t' = (1+\eta)^t w_0',
\]
and in particular $|w_T-w_T'| = (1+\eta)^T |w_0-w_0'| \geq \varepsilon e^{\eta T/2}$, for all $0 \le t \le T$.
To see that both trajectories remain inside $\dom$ note that for all $t$, 
$$w_t = (1+\eta)^tw_0 \leq e^{\eta t}w_0 <\frac{w_0}{w_0+\varepsilon}  \leq 1$$
And similarly,
$$w_t' = (1+\eta)^tw_0' \leq e^{\eta t}w_0' <\frac{w_0'}{w_0+\varepsilon} = 1,$$
which concludes the proof.
\end{proof}

\section{Proofs for Section~\ref{sec:euclidean}}
\subsection{Proof of Theorem~\ref{thm:endpoint_exponential_origin_local_feasible_set}}\label{sec:euclidean-proof}
We use \(a\lesssim b\) and \(a\gtrsim b\) to denote inequalities up to
absolute multiplicative constants. 
The proof has four steps. First we construct the feasible set, the regularizer,
and the objective. Second, we define a sequence \(\{w_t\}_{t=0}^T\) starting at
the origin and show that it is exactly generated by MD.
Third, we prove that the linearized dynamics along this sequence have an
expanding transverse direction. Finally, we initialize a second trajectory in
that expanding direction and control the nonlinear error by a quadratic
bootstrap.
\paragraph{Construction.}
We will first describe the construction itself. Fix \(0<\eta\le\eta_0\) and \(T\ge1\). The absolute constant \(\eta_0>0\) and
the final admissible range \(0<\varepsilon\le\varepsilon_0\) will be chosen at
the end of the proof. Let \(0<\delta<0.02\), \(0<\gamma\le 1/221\) be absolute constants. Let \(\sigma\in(0,1]\), to be chosen
later as a function of \(\varepsilon,\eta,T\). Set
\[
    \tilde\eta:=\frac{4}{5} \cdot \frac{1}{65}\cdot \eta,
    \qquad
    \omega:=1+\sigma\tilde\eta T.
\]
Define
\[
    r:=\omega^{-3/2},
    \qquad
    m:=\gamma\omega^{-3/2},
    \qquad
    z_t:=t\frac{\sigma\tilde\eta}{24\omega},
    \qquad
    0\le t\le T.
\]
The reference trajectory will be 
\[
    w_t:=(0,z_t)\in\mathbb R^2\times\mathbb R.
\]
The feasible set is
\[
    \mathcal W
    :=
    \left\{
    (x,z)\in\mathbb R^2\times\mathbb R:
    \|x\|_2\le r,\;
    -m\le z\le 2z_T+m
    \right\}.
\]
Therefore
\[
    \operatorname{dist}(w_t,\mathcal W^c)
    \ge
\frac{\gamma}{2}\omega^{-3/2},
    \qquad
    0\le t\le T.
\]
Moreover,
\[
    2z_T
    =
    \frac{\sigma\tilde\eta T}{12(1+\sigma\tilde\eta T)}
    \le
    \frac1{12},
    \qquad
    2m\le 2\gamma.
\]
Thus, after decreasing \(\gamma\) if necessary,
\[
    \operatorname{diam}(\mathcal W)
    \le
    \sqrt{(2r)^2+(2z_T+2m)^2}
    \le
    \sqrt{4+\left(\frac1{12}+2\gamma\right)^2}
    \le3.
\]

Let
\[
    P(\theta):=
    \begin{pmatrix}
    \cos\theta&-\sin\theta\\
    \sin\theta&\cos\theta
    \end{pmatrix},
    \qquad
    B:=
    \begin{pmatrix}
    3&0\\
    0&1
    \end{pmatrix},
    \qquad
    Q:=
    \begin{pmatrix}
    \frac{11}{10}&1\\
    1&\frac{11}{10}
    \end{pmatrix}.
\]
Define
\[
    M(z):=P(\omega z)BP(\omega z)^\top,
    \qquad
    H(z):=P(\omega z)QP(\omega z)^\top.
\]
For simplicity of notation we also sot
\[
    h:=\frac{9\sigma}{\omega},
    \qquad
    \lambda:=25+2\tilde\eta h T,
    \qquad 
    \mu:=
    \frac{\sigma}{24\omega}
    \left(\lambda-\frac{\tilde\eta h}{2}\right).
\]
The clock part of the regularizer is
\[
    R_{\rm clock}(z)
    :=
    \frac \lambda2z^2-\frac{4h\omega}{\sigma}z^3.
\]
Define
\[
    \widetilde R(x,z):=
    \frac12x^\top M(z)x+R_{\rm clock}(z),
\]
and
\[
    \widetilde F(x,z):=
    -\mu z+\frac h2z^2+\frac{\sigma}{12}x^\top H(z)x.
\]
Finally set
\[
    R:=\frac54\widetilde R,
    \qquad
    F:=\frac{1}{65}\widetilde F,
\]
and restrict these functions to \(\mathcal W\). The variable \(z\) plays the role of a clock. Along the reference trajectory
the \(x\)-coordinate is zero, while the matrices \(M(z)\) and \(H(z)\) rotate
with angular speed \(\omega\). The transverse instability will come from this
rotating geometry.

\paragraph{Auxiliary Lemmas.}
We now record the properties of the construction that will be used in the proof
of the theorem. The proofs of the following lemmas are deferred to \cref{sec:aux-lemmas-proof}.

The first lemma verifies the regularity assumptions: \(R\) is
uniformly strongly convex and smooth, while \(F\) is convex, smooth, and
Lipschitz.
\begin{lemma}[Regularity of the construction]
\label{lem:regularity_endpoint_final}
The functions \(R,F\) satisfy, on \(\mathcal W\),
\[
    I \preceq \nabla^2R\preceq C_0I,
    \qquad
    0 \preceq \nabla^2 F \preceq I
\]
Also $F$ is $1$-Lipschitz. Moreover, let $0< \delta< 0.02$, on 
\[
    U :=
    \left\{
    (x,z):
    \|x\|_2<(1+\delta)r,\;
    -(1+\delta)m<z<2z_T+(1+\delta)m
    \right\},
\]
it holds that there exists some absolute constant $C_1>0$ such that
\[
    \|D^3R\|\le C_1\omega,
    \qquad
    \|D^4R\|\le C_1\omega^2,
\]
and
\[
    \|D^2F\|\le C_1\sigma,
    \qquad
    \|D^3F\|\le C_1\sigma\omega.
\]
\end{lemma}
The next lemma verifies that the parameters in the clock coordinate were chosen
so that the reference sequence \(\{w_t\}_{t=0}^T\) is exactly generated by
MD updates.
\begin{lemma}[Exact reference trajectory]
\label{lem:exact_endpoint_final}
The points \(w_t=(0,z_t)\), \(0\le t\le T\), form an exact MD trajectory
\end{lemma}
We next study the derivative of the mirror descent update along this reference
trajectory. The following lemma shows that, in a rotating frame, the transverse
linearized dynamics reduce to repeated multiplication by a single matrix
\(A_\sigma\), which has one expanding eigenvalue.
\begin{lemma}[Exponential transverse Jacobian]
\label{lem:linearized_endpoint_final}
Let \(\Phi\) denote the local MD update map,
\[
    \Phi(y):=(\nabla R)^{-1}
    \bigl(\nabla R(y)-\eta\nabla F(y)\bigr),
\]
and set
\(
    \theta_t:=\omega z_t.
\)
For \(0\le t\le T\), define
\[
    \mathcal J_t:=D\Phi(w_{t-1})\cdots D\Phi(w_0),
    \qquad
    \mathcal J_0:=I.
\]
Then the clock block of \(\mathcal J_t\) is \(1\), and its transverse block is
\[
    P(\theta_t)A_\sigma^t,
    \quad \text{where} \quad
    A_\sigma
    =
    B^{-1}P\left(-\frac{\sigma\tilde\eta}{24}\right)
    \left(B-\frac{\sigma\tilde\eta}{6}Q\right).
\]
For all \(\eta \leq \eta_0\) for a sufficiently small absolute constant $\eta_0$, the matrix \(A_\sigma\) has an expanding
eigenvalue \(\rho>1\) satisfying
\[
    c_{\rho}\sigma\eta\le \log\rho\le C_{\rho}\sigma\eta,
    \qquad
    \rho-1\ge c_{\rho}\sigma\eta.
\]
For some absolute constants $c_\rho, C_\rho >0$.
Moreover, the eigenbasis of \(A_\sigma\) corresponding to its two real
eigenvalues is uniformly well-conditioned.
\end{lemma}
It remains to control the nonlinear error around the reference trajectory. We
do this in two steps. First, we use a general bound on the second derivative of
the mirror descent update map. Then we apply that bound to the present
construction in the rotating frame.
\begin{lemma}[Second derivative of the M update]
\label{lem:second_derivative_update}
Let \(U\subset\mathbb R^d\) be open, and let \(R,F\in C^4(U)\). Suppose that
on \(U\),
\[
    \nabla^2R\succeq I,
    \qquad
    \|D^3R\|\le M_3,
    \qquad
    \|D^4R\|\le M_4,
\]
and
\[
    \|D^2F\|\le N_2,
    \qquad
    \|D^3F\|\le N_3.
\]
Let
\[
    \Phi(y):=(\nabla R)^{-1}\bigl(\nabla R(y)-\eta\nabla F(y)\bigr)
\]
be locally defined near \(y\in U\). Assume that
\[
    \Phi(y)\in U,
    \qquad
    [y,\Phi(y)]\subset U,
    \qquad
    \|\Phi(y)-y\|\le \Delta .
\]
Then
\[
    \|B\Phi(y)-I\|
    \le
    M_3\Delta+\eta N_2,
\]
and
\[
    \|D^2\Phi(y)\|
    \le
    2\Bigl(
        M_4\Delta
        +\eta N_3
        +M_3(M_3\Delta+\eta N_2)
        +M_3(M_3\Delta+\eta N_2)^2
    \Bigr).
\]
\end{lemma}
We now introduce the rotating-frame norm used in the nonlinear estimate. Let
\(u_+\) and \(u_-\) be eigenvectors of \(A_\sigma\) with
\[
    A_\sigma u_+=\rho u_+,
    \qquad
    A_\sigma u_-=\rho_-u_-,
\]
where \(\rho>1\) is the expanding eigenvalue from
Lemma~\ref{lem:linearized_endpoint_final}. Normalize \(u_+\) so that
\(\|u_+\|_2=1\), and set
\[
    E_\sigma:=\begin{pmatrix}u_+&u_-\end{pmatrix}.
\]
By Lemma~\ref{lem:linearized_endpoint_final}, this eigenbasis is uniformly
well-conditioned. Thus, for some absolute constant \(\ell\ge1\),
\[
    \|E_\sigma\|,\|E_\sigma^{-1}\|\le \ell.
\]
For \(Y=(y,s)\in\mathbb R^2\times\mathbb R\), write
\[
    E_\sigma^{-1}y=
    \begin{pmatrix}\alpha\\ \beta\end{pmatrix},
\]
and define
\[
    \|Y\|_\#:=|\alpha|+|\beta|+\ell |s|.
\]
Then \(\|\cdot\|_\#\) is uniformly equivalent to the Euclidean norm and
\[
    \|(u_+,0)\|_\#=1.
\]
Finally, define the one-step linearized map in the rotating frame by
\[
    L(y,s):=(A_\sigma y,s).
\]
Then
\(
    \|LY\|_\#\le \rho\|Y\|_\#.
\)
The following lemma gives the desired quadratic control of the nonlinear
remainder in this rotating frame.
\begin{lemma}[Nonlinear control]
\label{lem:nonlinear_endpoint_final}
Let
\(
    b:=\varepsilon\rho^T.
\)
There are absolute constants \(K,c_b>0\) such that, if
\(
    b\le c_b\omega^{-3/2},
\)
then the following holds. If
\[
    Y_t:=\mathcal T_t(w'_t-w_t),
    \quad \text{and} \quad
    [w_t,w'_t]\subset B(w_t,Kb),
\]
then there exists some absolute constant $C_2>0$ such that
\[
    \|Y_{t+1}-LY_t\|_\#
    \le
    C_2\sigma\eta\omega^2(1+\sigma\eta\omega)\|Y_t\|_\#^2.
\]
\end{lemma}
Finally, we use the following elementary bootstrap lemma to turn the one-step
quadratic error estimate into a uniform-in-time bound.
\begin{lemma}[Quadratic Bootstrap] \label{lem:quadratic_bootstrap}
Let $(q_t)_{t=0}^T$ be a sequence of nonnegative numbers with $q_0=1$, and suppose
\[
q_t\le q_{t-1}+a_t q_{t-1}^2
\qquad\text{for }1\le t\le T,
\]
where $a_t\ge 0$ and
\[
\sum_{t=1}^T a_t\le \frac14.
\]
Then
\[
q_t\le 2
\qquad\text{for all }0\le t\le T.
\]
\end{lemma}
We are now ready to prove the theorem.
\begin{proof}[Proof of Theorem~\ref{thm:endpoint_exponential_origin_local_feasible_set}]
We use the construction above, with \(\sigma\in(0,1]\) to be chosen below. By
Lemma~\ref{lem:regularity_endpoint_final}, the functions \(R,F\) have all the
required regularity properties on \(\mathcal W\), and
\(\operatorname{diam}(\mathcal W)\le3\). By
Lemma~\ref{lem:exact_endpoint_final}, \(w_t=(0,z_t)\) is an exact MD trajectory.

Let \(u_+\) be a unit expanding eigenvector of \(A_\sigma\), and set
\[
    v:=(u_+,0),
    \qquad
    w'_0:=\varepsilon v.
\]
For this trajectory, define
\[
    Y_t:=\mathcal T_t(w'_t-w_t),
    \qquad 0\le t\le T.
\]
from the notations above. Let
\(
    b:=\varepsilon\rho^T.
\)
Assume for the moment that
\(
    b\le c_b\omega^{-3/2},
\)
for small enough $c_b$ such that $4\ell b \leq \frac{\gamma}{2}\omega^{-3/2}$, where $\ell$ is the constant of the norm equivalence, $\|Y\|_2 \leq \sqrt{2}\ell\|Y\|_\#$. We will show at the end of the proof that we choose $\sigma$ such that this will hold. We will now prove by induction that:
\[
    \|Y_t\|_\#\le2\varepsilon\rho^t, \quad  [w_t, w_t'] \subset B(w_t, 4\ell b)  \qquad
    0\le t\le T.
\] 
For $t=0$ we have $\|Y_t\|_\#, \|w_0 -w_0'\|_2 =\varepsilon \leq b < 4\ell b$. Now assume this holds for all $0 \leq k \leq t-1$ and we will prove for $t$.
Using the inductive assumption Lemma~\ref{lem:nonlinear_endpoint_final}
gives for every $k \in [t-1]$
\[
    \|Y_{k+1}-LY_k\|_\#
    \le
    C_2\sigma\eta\omega^2(1+\sigma\eta\omega) \|Y_k\|_\#^2.
\]
Set
\[
    q_k:=\frac{\|Y_k\|_\#}{\varepsilon\rho^k}.
\]
Since \(\|(u_+,0)\|_\#=1\), we have \(q_0=1\). The nonlinear recurrence gives
\[
    q_{k}
    \le
    q_{k-1}+
    C_2\sigma\eta\left(\omega^2(1+\sigma\eta\omega) \right)\varepsilon\rho^{k-2}q_{k-1}^2.
\]
Since \(\rho\ge1\), we may weaken this to
\[
    q_{k}
    \le
    q_{k-1}+
    C_2\sigma\eta\left(\omega^2(1+\sigma\eta\omega) \right)\varepsilon\rho^{k-1} q_{k-1}^2.
\]
Befine $a_{k} := C\sigma\eta\left(\omega^2(1+\sigma\eta\omega) \right)\varepsilon\rho^{k-1}$ we have:
\[
    q_{k}
    \le
    q_{k-1}+
    a_kq_{k-1}^2.
\]
Moreover,
\[
    \sigma\eta\sum_{k=1}^{t}\varepsilon\rho^{k-1}
    = \sigma\eta\frac{\varepsilon\rho^t}{\rho-1}
    \lesssim b,
\]
because \(b=\varepsilon\rho^T\) and \(\rho-1\ge c_{\rho}\sigma\eta\). Hence there exists an absolute constant $C_3 > 0$ such that,
\[
    \sum_{k=1}^t a_k = \sum_{k=1}^{t}
    C_2\sigma\eta\left(\omega^2(1+\sigma\eta\omega) \right)\varepsilon\rho^{k-1}
    \le
    C_3\omega^2(1+\sigma\eta\omega)b.
\]
If \(C_3\left(\omega^2(1+\sigma\eta\omega) \right) b \leq \frac{1}{4}\) (we will choose $\sigma$ later so that this will hold) , \cref{lem:quadratic_bootstrap} gives
\(
    q_k\le2 \text{ for }
    0\le k\le t.
\)
Thus
\[
    \|Y_{t}\|_\#\le2\varepsilon\rho^t
\]
In particular, by norm equivalence and since rotations are isometric and $\|Y\|_2 \leq \sqrt{2}\ell \|Y\|_\#$,
\[
    \|w'_{t}-w_{t}\|_2\le 2\sqrt{2}\ell b < 4\ell b.
\]
Hence, \([w_t,w_t'] \subset B(w_t, 4\ell b)\). This concludes the induction. Thus we have shown for all $0 \leq t \leq T$, 
\[
    \|Y_t\|_\# \leq 2\varepsilon \rho^t, \quad \|w_t - w_t'\|_2 < 4\ell b 
\]
Since
\(
    \operatorname{dist}(w_t,\mathcal W^c)\ge \frac{\gamma}{2}\omega^{-3/2},
\)
the condition \(b\le c_b\omega^{-3/2}\), for small enough $c_b$, ensures that
\(w'_t\in\mathcal W\) for all \(0\le t\le T\).
For the endpoint comparison, note that
\[
    Y_T-L^TY_0
    =
    \sum_{j=0}^{T-1}L^{T-1-j}\left(Y_{j+1}-LY_j\right).
\]
Using \( \|L Y\|_\# \leq \rho \|Y\|_\#\) and \(\|Y_j\|_\#\le2\varepsilon\rho^j\), we get
\begin{align*}
     \|Y_T-L^TY_0\|_\#
    &\le C_3\sigma\eta\omega^2(1+\sigma\eta\omega)\varepsilon^2
    \sum_{j=0}^{T-1}L^{T-1-j}\|Y_j\|_\#^2 \leq 4C_3\sigma\eta\omega^2(1+\sigma\eta\omega)\varepsilon^2
    \sum_{j=0}^{T-1}\rho^{T-1-j}\rho^{2j}\\
    &\le4C_3\sigma\eta\omega^2(1+\sigma\eta\omega)\varepsilon^2
    \sum_{j=0}^{T-1}\rho^{T-1+j} = 4C_3\sigma\eta\omega^2(1+\sigma\eta\omega)\varepsilon^2
    \cdot \frac{\rho^{2T-1}}{1-\rho}\\
    &\leq 4C_3\sigma\eta\omega^2(1+\sigma\eta\omega)\varepsilon^2
    \cdot \frac{\rho^{2T}}{1-\rho} \tag{$\rho>1$}
\end{align*}
Since \(\rho-1\ge c_{\rho}\sigma\eta\), there exists an absolute constant $C_5>0$ s.t.
\begin{align*}
    \|Y_T-L^TY_0\|_\#  =
    C_5\omega^2(1+\sigma\eta\omega) \varepsilon^2\rho^{2T}
    =
    C_5\omega^2(1+\sigma\eta\omega)  b^2
\end{align*}
Again, from the norm equivalence there exists an absolute constant $C_6 > 0$ such that:
\[
    \|Y_T-L^TY_0\|_2
    \le
    C_6\left(\omega^2(1+\sigma\eta\omega) \right) b^2.
\]
Since \(Y_0=\varepsilon(u_+,0)\), we have \(\|L^TY_0\|_2=b\). Therefore, if
\[
    C_6\left(\omega^2(1+\sigma\eta\omega) \right) b\le\frac12,
\]
then
\[
    \|w'_T-w_T\|_2 = \|Y_T\|_2 \geq \|L^TY_0\|_2 - \|Y_T - L^TY_0\| \ge\frac12 b.
\]
It remains to choose \(\sigma\) so that \(b\) has the desired size and the following two
smallness conditions hold:
\[
b \leq c_b \omega^{-3/2} \quad \text{and} \quad \left(\max\{C_6,2C_3\}\cdot\omega^2(1+\sigma\eta\omega)  b\right)\le\frac12.
\]
Set
\[
    L_\varepsilon:=1+\log\frac1\varepsilon,
    \qquad
    B_\varepsilon:=\frac{a}{L_\varepsilon^3},
\]
where \(a>0\) is a sufficiently small absolute constant that will be chosen later. For a fixed $a$, Let $\varepsilon_0(a)$ be such that $\varepsilon \leq B_{\varepsilon}$ for any $\varepsilon\leq \varepsilon_0(a)$. Indeed there exists such a constant since $ \varepsilon\left(1+\log\frac1\varepsilon\right)^3\downarrow0$ as $\varepsilon\downarrow0$. 
We choose
\(\sigma\in(0,1]\) so that
\[
    \varepsilon\rho(\sigma)^T
    =
    \min\{B_\varepsilon,\varepsilon e^{c\eta T}\}.
\]
for some absolute constant $c>0$. This choice is possible by the intermediate value theorem. Indeed,
\(\rho(\sigma)\) depends continuously on \(\sigma\), and
\(\rho(\sigma)\downarrow1\) as \(\sigma\downarrow0\). Thus
\[
    \varepsilon\rho(\sigma)^T\to\varepsilon.
\]
For $\varepsilon \leq \varepsilon_0(a)$, 
\[
    B_\varepsilon\ge\varepsilon,
\]
On the other hand, when \(\sigma=1\), the lower bound
\(\log\rho(1)\ge c_{\rho}\eta\) gives
\[
    \varepsilon\rho(1)^T\ge\varepsilon e^{c\eta T},
\]
after decreasing the absolute constant \(c>0\) if necessary. Hence such a
\(\sigma\in(0,1]\) exists.
For this choice,
\[
    T\log\rho
    =
    \log\frac{b}{\varepsilon}
    \le
    \log\frac1\varepsilon.
\]
Since \(\sigma\eta\lesssim \log\rho\), we obtain
\[
    \sigma\eta T\lesssim L_\varepsilon.
\]
Therefore
\[
    \omega=1+\sigma\tilde\eta T \lesssim  L_\varepsilon.
\]
Consequently,
\[
    \omega^2(1+\sigma\eta\omega)
    \lesssim L_\varepsilon^3.
\]
Because
\[
    b\le B_\varepsilon=aL_\varepsilon^{-3},
\]
we get that there exists an absolute constant $C_7> 0$ such that 
\[
    \max\{C_6,2C_3\} \cdot \left(\omega^2(1+\sigma\eta\omega) \right) b\le C_7a.
\]
Choosing \(a\leq \frac{1}{2C_7}\) gives
\[
    \max\{C_6,2C_3\}\cdot \left(\omega^2(1+\sigma\eta\omega) \right) b\le\frac14.
\]
Finally, since \(\omega\lesssim L_\varepsilon\)  there exists some absolute constant $C_8$ such that,
\[
    L_\varepsilon^{-3/2} \leq C_8 \omega^{-3/2}.
\]
On the other hand,
\[
    b\le aL_\varepsilon^{-3} \leq aL_\varepsilon^{-3} \leq aL_\varepsilon^{-3/2} \leq C_8a\omega^{-3/2}.
\]
choosing $a \leq \frac{c_b}{C_8}$ would ensure 
\[
    b \leq c_b\omega^{-3/2}
\]
Thus choosing the following absolute constants, 
\[
    a = \min \left\{\frac{1}{e}, \frac{c_b}{C_8},\frac{1}{4C_7} \right\} , \quad \varepsilon_0 = \varepsilon_0(a)
\]
and the $\sigma$ that follows ensures both smallness conditions hold. Therefore
\[
    \|w'_T-w_T\|_2
    \ge
    \frac12 b
    =
    \frac12
    \min\{B_\varepsilon,\varepsilon e^{c\eta T}\}.
\]
Hence,
\[
    \|w'_T-w_T\|_2
    = \Omega \left(\min\left\{
        \frac1{(1+\log(1/\varepsilon))^3},
        \varepsilon e^{c\eta T}
    \right\}\right).
\]
This proves the theorem. 
\end{proof}

\subsection{Proofs of Auxiliary Lemmas}\label{sec:aux-lemmas-proof}
For the proof of the Auxiliary lemmas we will use the following lemmas:

\begin{lemma}[Quantitative inverse branch]
\label{lem:quantitative_inverse_branch_endpoint_local}
Let \(U\subseteq\mathbb R^d\) be open and convex, and let
\(R\in C^2(U)\) satisfy
\[
    \nabla^2R(x)\succeq I
    \qquad
    \forall x\in U.
\]
Let \(u_0\in U\), and set
\[
    \tau:=\operatorname{dist}(u_0,U^c)>0.
\]
Then
\[
    B\left(\nabla R(u_0),\frac \tau4\right)
    \subseteq
    \nabla R\bigl(B(u_0,\tau/2)\bigr).
\]
Moreover, the inverse branch of \(\nabla R\) on this ball is \(1\)-Lipschitz:
if
\(
    u_i\in B(u_0,\tau/2) \text{ for }
    i=1,2,
\)
then
\[
    \|u_1-u_2\|_2\le \|\nabla R(u_1)- \nabla R(u_2)\|_2.
\]
\end{lemma}

\begin{proof}[Proof of the \cref{lem:quantitative_inverse_branch_endpoint_local}]
Fix
\(
    p\in B\left(\nabla R(u_0),\frac \tau4\right).
\)
Consider
\[
    \varphi_p(u):=R(u)-\langle p,u\rangle
\]
on the compact ball \(\overline B(u_0,\tau/2)\subset U\). Let \(u_p\) be a
minimizer of \(\overline B(u_0,\tau/2)\) that is $\varphi_p$. If $u$ is a boundary point of  \(\|u-u_0\|_2=\tau/2\), then strong convexity gives
\begin{align*}
    \varphi_p(u)-\varphi_p(u_0)
    &\ge
    \langle\nabla R(u_0)-p,u-u_0\rangle
    +
    \frac12\|u-u_0\|_2^2 \\
    &\ge
    -\|\nabla R(u_0)-p\|_2\|u-u_0\|_2
    +
    \frac12\|u-u_0\|_2^2 \\
    & >  \frac{\tau}{4} \cdot \frac{\tau}{2} + \frac{1}{2} \cdot \left( \frac{\tau}{2}\right)^2
    =0.
\end{align*}
Thus \(u_p\) is interior, and hence
\[
    \nabla R(u_p)=p.
\]
This proves the image inclusion.
For the Lipschitz estimate, strong monotonicity gives
\[
    \|u_1-u_2\|_2^2\leq  \langle\nabla R(u_1)-\nabla R(u_2),u_1-u_2\rangle
    \leq \|\nabla R(u_1)-\nabla R(u_2)\|_2\|u_1-u_2\|_2.
\]
Therefore
\[
    \|u_1-u_2\|_2\le \|\nabla R(u_1)-\nabla R(u_2)\|_2.
\]
\end{proof}

\subsubsection{Proof of Lemma~\ref{lem:regularity_endpoint_final}}

First we note that the rotating blocks satisfy
\[
    I\preceq M(z)\preceq3I,
    \qquad
    \frac1{10}I\preceq H(z)\preceq\frac{21}{10}I.
\]
Moreover the $j$-th derivative satisfies,
\[
    \|M^{(j)}(z)\|+\|H^{(j)}(z)\|
    \le c_j\omega^j,
    \qquad
    j\ge1.
\]
For some constants $c_j$.
On \(U\), we have
\[
    \|x\|_2
    \le
    (1+\delta)\omega^{-3/2}.
\]
From direct calculation,
\[
    \|M'(z)x\|_2\le 2\omega \cdot (1+\delta)\omega^{-3/2} = 2(1+\delta) \omega^{-1/2},
\]
and
\[
    |x^\top M''(z)x|\le 4\omega^2 \cdot (1+\delta)^2\omega^{-3} = 4(1+\delta)^2\omega^{-1}.
\]
The Hessian of \(\widetilde R\) is
\[
    \nabla^2\widetilde R(x,z)
    =
    \begin{pmatrix}
    M(z)&M'(z)x\\
    x^\top M'(z)&R_{\rm clock}''(z)+\frac12x^\top M''(z)x
    \end{pmatrix}.
\]
On \(U\),
\[
    z\le 2z_T+(1+\delta)m,
    \qquad
    2z_T=\frac{\sigma\tilde\eta T}{12\omega}.
\]
Since \(h=9\sigma/\omega\), the main part of the clock curvature cancels:
\begin{align*}
     R_{\rm clock}''(z) &= \lambda - \frac{24 h \omega}{\sigma}z \geq \lambda - \frac{24 h \omega}{\sigma}(2z_T + (1+\delta)m)
\end{align*}
Since,
\[ 
    \frac{24h\omega}{\sigma}\cdot 2z_T
    =
    2\tilde\eta h T.
\]
This cancels the \(2\tilde\eta h T\) contribution in
\(\lambda=25+2\tilde\eta h T\). The extra loss from the padding is at most
\[
    \frac{24h\omega}{\sigma}(1+\delta)m
    =
    216(1+\delta)\gamma\omega^{-3/2}
    \le
    216(1+\delta)\gamma.
\]
From the choice of $\gamma$, we get
\[
    R_{\rm clock}''(z)
    =
    \lambda-\frac{24h\omega}{\sigma}z
    \ge24
    \qquad
    \text{on }U.
\]
Therefore, for \((y,s)\in\mathbb R^2\times\mathbb R\),
\begin{align*}
    \left\langle
    \nabla^2\widetilde R(x,z)(y,s),(y,s)
    \right\rangle & = y^\top M(z) y + 2sy^\top M'(z)x + s^2\left( R_{\rm clock}''(z)+\frac12x^\top M''(z)x\right)
    \\
    & \geq
    \|y\|_2^2
    -4(1+\delta)\omega^{-1/2}|s|\|y\|_2
    +(24-2(1+\delta)^2\omega^{-1})s^2.
\end{align*}
Now notice that:
\[
    -4(1+\delta)\omega^{-1/2}|s|\|y\|_2 \geq - \frac{1}{5} \|y\|_2^2 - 20(1+\delta)^2\omega^{-1}s^2
\]
Plugging that back,
\begin{align*}
    \left\langle
    \nabla^2\widetilde R(x,z)(y,s),(y,s)
    \right\rangle & \geq 
    \|y\|_2^2
    -\frac{1}{5} \|y\|_2^2 - 20(1+\delta)^2\omega^{-1}s^2
    +(24-2(1+\delta)^2\omega^{-1})s^2\\
    & \geq  \frac{4}{5} \|y\|_2^2
    +(24-22(1+\delta)^2\omega^{-1})s^2 \\
    &\geq  \frac{4}{5} \|y\|_2^2
    +\frac{4}{5}s^2 \tag{$\omega \geq 1$, choice of $\delta$}
\end{align*}

Hence
\[
    \nabla^2\widetilde R\succeq\frac45I.
\]
Since \(R=(5/4)\widetilde R\), this gives
\[
    \nabla^2R\succeq I.
\]
For the upper bound,
\[
    R_{\rm clock}''(z)+\frac12x^\top M''(z)x \leq  25 + 2\tilde\eta h T + 216(1+\delta)\gamma \leq 25 + 18 + 216(1+\delta)\gamma
\]
which is bounded by an absolute constant. We get
\begin{align*}
    \|\nabla^2 \widetilde R(x,z)\| & \leq \|M(z)\| + 2\|M'(z)x\| + |R_{\rm clock}''(z)+\frac12x^\top M''(z)x| \\
    & \leq 4(1+\delta)^2\omega^{-1} + 4(1+\delta)\omega^{-1/2} + |R_{\rm clock}''(z)+\frac12x^\top M''(z)x|
\end{align*}
Hence there exists an absolute constant $C_0$, such that 
\[
    \nabla^2R \preceq C_0I
\]
We now prove convexity of \(\widetilde F\) with explicit constants. Its Hessian is
\[
    \nabla^2\widetilde F(x,z)
    =
    \begin{pmatrix}
    \frac\sigma6H(z)&\frac\sigma6H'(z)x\\
    \frac\sigma6x^\top H'(z)&
    h+\frac\sigma{12}x^\top H''(z)x
    \end{pmatrix}.
\]
Because the eigenvalues of \(Q\) are \(1/10\) and \(21/10\),
\[
    \frac1{10}I\preceq H(z)\preceq\frac{21}{10}I.
\]
Also, directly from the rotating form,
\[
    \|H'(z)\|\le2\omega,
    \qquad
    \|H''(z)\|\le4\omega^2.
\]
The upper-left block satisfies
\[
    \frac\sigma6H(z)
    \succeq
    \frac{\sigma}{60}I,
\]
so its inverse has operator norm at most \(60/\sigma\). The off-diagonal block
satisfies
\[
    \left\|\frac\sigma6H'(z)x\right\|_2
    \le
    \frac\sigma6\cdot2\omega\cdot(1+\delta)\omega^{-3/2}
    =
    \frac{1+\delta}{3}\sigma\omega^{-1/2}.
\]
Therefore the Schur-complement loss is at most
\begin{align*}
      &\left(\frac\sigma6x^\top H'(z)\right)\left(\frac\sigma6H(z)\right)^{-1}\left(\frac\sigma6H'(z)x\right) \leq \frac6\sigma\|H(z)^{-1}\| \cdot \frac\sigma6  \| H'(z)x\|_2^2\\
      & \leq \frac{60}{\sigma} \left( \frac{(1+\delta) \omega^{-1/2} \sigma}{3} \right)^2 = \frac{20}{3}(1+\delta)^2\frac{\sigma}{\omega}.
\end{align*}
For the lower-right block, using \(h=9\sigma/\omega\),
\begin{align*}
    h+\frac\sigma{12}x^\top H''(z)x
    &\ge
    \frac{9\sigma}{\omega}
    -
    \frac\sigma{12}\cdot4\omega^2\cdot(1+\delta)^2\omega^{-3} \\
    &=
    \left(9-\frac{(1+\delta)^2}{3}\right)
    \frac{\sigma}{\omega}.
\end{align*}
Thus the Schur complement is at least
\[
    \left(
        9-\frac{(1+\delta)^2}{3}
        -\frac{20}{3}(1+\delta)^2
    \right)
    \frac{\sigma}{\omega}
    =
    \left(9-7(1+\delta)^2\right)
    \frac{\sigma}{\omega}.
\]
By the choice of \(\delta\), this is nonnegative. Hence
\[
    \nabla^2\widetilde F\succeq0
    \qquad
    \text{on }U.
\]

The same block estimates give the sharper bound
\begin{align*}
    \|\nabla^2 \widetilde F(x,y)\| & \leq \|\frac\sigma6H(z)\| + 2\|\frac\sigma6H'(z)x\|_2 +|h+\frac\sigma{12}x^\top H''(z)x| \\
    & \leq \frac{7\sigma}{20} + 2\cdot \frac{(1+\delta)\sigma \omega^{-1/2}}{3} + 9\sigma + \frac{(1+\delta)^2\sigma \omega^{-1}}{3} \\
    & \leq \left(\frac{7}{20} + \frac{4}{3} + 9 + \frac{2}{3}\right) \sigma \leq 12\sigma  \tag{$\delta \leq 1 \leq  \omega $} 
\end{align*}
Therefore
\[
    \|D^2F\|
    =
    \frac{1}{65}\|\nabla^2\widetilde F\|
    \le
    12\frac{1}{65}\sigma.
\]

Next,
\[
    \|\nabla_x\widetilde F(x,z)\|_2 =\|\frac\sigma6H(z)x\| \leq \frac{(1+\delta)\sigma w^{-1/2}}{3} \leq \frac{2}{3}\sigma.
\]
Also,
\begin{align*}
     |\partial_z\widetilde F(x,z)|
    &=
    |-\mu+
    h z+
    \frac\sigma{12}x^\top H'(z)x| \\
    & \leq\frac{\sigma}{24 \omega} \left(25 + 2 \tilde\eta h T - \frac{\tilde h}{2} \right) + 9\sigma(2z_T + (1+\delta)\gamma \omega^{-3/2}) + \frac{(1+\delta)\omega^{-1/2}\sigma}{3}\\
    & \leq \sigma(25 + 18) + 9\sigma\left(\frac{2}{24} + 2\right) + \frac{2}{3} \sigma \leq 63 \sigma .
\end{align*}
Hence
\[
    \|\nabla\widetilde F\|_2\le 65\sigma.
\]
Since \(\frac{1}{65}< \frac{1}{65}\), \(F=\frac{1}{65}\widetilde F\) convex,
\(1\)-Lipschitz, and \(1\)-smooth.
Finally, we justify the higher derivative bounds. Since
\[
\widetilde R(x,z)=\frac12x^\top M(z)x+R_{\rm clock}(z),
\]
and the first term is quadratic in \(x\), every derivative with more than two
\(x\)-derivatives vanishes. The nonzero third derivatives coming from
\(\frac12x^\top M(z)x\) are bounded by
\[
\|M'(z)\|\lesssim \omega,\qquad
\|M''(z)x\|_2\lesssim\omega^2\omega^{-3/2}=\omega,
\]
and
\[
|x^\top M^{(3)}(z)x|
\lesssim\omega^3\omega^{-3}=\omega.
\]
The nonzero fourth derivatives are bounded by
\[
\|M''(z)\|\lesssim\omega^2,\qquad
\|M^{(3)}(z)x\|_2\lesssim\omega^3\omega^{-3/2}=\omega^2,
\]
and
\[
|x^\top M^{(4)}(z)x|
\lesssim\omega^4\omega^{-3}=\omega^2.
\]
Also,
\[
R_{\rm clock}^{(3)}(z)
=
-\frac{24h\omega}{\sigma}
=
-216,
\qquad
R_{\rm clock}^{(4)}(z)=0.
\]
Thus for and absolute constant $C_R$ large enough, 
\[
\|D^3R\|\le C_R\omega,
\qquad
\|D^4R\|\le C_R\omega^2.
\]

Similarly,
\[
\widetilde F(x,z)
=
-\mu z+\frac h2z^2+\frac{\sigma}{12}x^\top H(z)x.
\]
The clock part is quadratic, so it contributes no third derivatives. The
nonzero third derivatives coming from the rotating block are bounded by
\[
\sigma\|H'(z)\|\lesssim\sigma\omega,
\qquad
\sigma\|H''(z)x\|_2
\lesssim\sigma\omega^2\omega^{-3/2}=\sigma\omega,
\]
and
\[
\sigma |x^\top H^{(3)}(z)x|
\lesssim\sigma\omega^3\omega^{-3}
=\sigma\omega.
\]
Since \(F=\frac{1}{65}\widetilde F\), we obtain for large enough $C_F$
\[
\|D^3F\|\le C_F\frac{1}{65}\sigma\omega .
\]
Taking $C_1= \max\{C_R, C_F\}$ concludes the proof.

\subsubsection{Proof of Lemma~\ref{lem:exact_endpoint_final}}
On the axis \(x=0\), the transverse gradients vanish. It is enough to check the
clock coordinate. Note that
\[
    z_{t+1}-z_t=\frac{\sigma\tilde\eta}{24\omega}.
\]
Since
\[
    R_{\rm clock}'(z)=\lambda z-\frac{12h\omega}{\sigma}z^2,
\]
we have
\begin{align*}
    R_{\rm clock}'(z_{t+1})-R_{\rm clock}'(z_t)
    &=
    (z_{t+1} - z_t)
    \left(
        \lambda-\frac{12h\omega}{\sigma}(z_{t+1}+z_t)
    \right) \\
    &=
    (z_{t+1} - z_t)
    \left(
        \lambda-\frac{24h\omega}{\sigma}z_t
        -\frac{12h\omega}{\sigma}(z_{t+1} - z_t)
    \right).
\end{align*}
Plugging in $(z_{t+1} - z_t) = \frac{\sigma \tilde\eta}{24\omega}$ and $\mu = \frac{\sigma}{24\omega}\left(\lambda - \frac{\tilde\eta h}{2} \right)$, this becomes
\[
    R_{\rm clock}'(z_{t+1})-R_{\rm clock}'(z_t)
    =
    \tilde\eta\mu-\tilde\eta h z_t.
\]
Since
\[
    \partial_z\widetilde F(0,z_t)=-\mu+h z_t,
\]
we get
\[
    \nabla\widetilde R(w_{t+1})
    =
    \nabla\widetilde R(w_t)-\tilde\eta\nabla\widetilde F(w_t).
\]
Finally, \(R=(5/4)\widetilde R\), \(F=\frac{1}{65}\widetilde F\), and
\(\tilde\eta=(4\frac{1}{65}/5)\eta\). Therefore
\[
    \nabla R(w_{t+1})
    =
    \nabla R(w_t)-\eta\nabla F(w_t).
\]

\subsubsection{Proof of Lemma~\ref{lem:linearized_endpoint_final}}

Bifferentiating the update gives
\[
    B\Phi(w_t)
    =
    \nabla^2R(w_{t+1})^{-1}
    \left(
        \nabla^2R(w_t)-\eta\nabla^2F(w_t)
    \right).
\]
Along the axis, the Hessians are block diagonal.

First consider the clock block. For the scaled functions, the clock block is
\[
    \frac{
        \frac54 R_{\rm clock}''(z_t)-\eta\frac{1}{65} h
    }{
        \frac54 R_{\rm clock}''(z_{t+1})
    }.
\]
Since
\[
    R_{\rm clock}''(z_{t+1})-R_{\rm clock}''(z_t)
    =
    -\frac{24h\omega}{\sigma}(z_{t+1}-z_t)
    =
    -\tilde\eta h,
\]
and
\[
    \eta\frac{1}{65}=\frac54\tilde\eta,
\]
we have
\[
    \frac54 R_{\rm clock}''(z_t)-\eta\frac{1}{65} h
    =
    \frac54\left(R_{\rm clock}''(z_t)-\tilde\eta h\right)
    =
    \frac54 R_{\rm clock}''(z_{t+1}).
\]
Thus the clock block is equal to \(1\).

The transverse block is
\[
    M(z_{t+1})^{-1}
    \left(
        M(z_t)-\tilde\eta\frac\sigma6H(z_t)
    \right).
\]
Since
\[
    M(z_t)=P(\theta_t)BP(\theta_t)^\top,
    \qquad
    H(z_t)=P(\theta_t)QP(\theta_t)^\top,
\]
and
\[
    \theta_{t+1}-\theta_t
    =
    \omega(z_{t+1}-z_t)
    =
    \frac{\sigma\tilde\eta}{24},
\]
the transverse block is
\begin{align*}
    P(\theta_{t+1})B^{-1}P(\theta_{t+1})^\top
    P(\theta_t)
    \left(
        B-\frac{\sigma\widetilde\eta}{6}Q
    \right)
    P(\theta_t)^\top  &=
    P(\theta_{t+1})B^{-1}
    P(\theta_t-\theta_{t+1})
    \left(
        B-\frac{\sigma\widetilde\eta}{6}Q
    \right)
    P(\theta_t)^\top \\
    P(\theta_{t+1})B^{-1}
    P\left(-\frac{\sigma \tilde\eta}{24}\right)
    \left(
        B-\frac{\sigma\widetilde\eta}{6}Q
    \right)
    P(\theta_t)^\top &=
    P(\theta_{t+1})A_\sigma P(\theta_t)^\top.
\end{align*}
Hence using telescoping terms the \(t\)-step linearized map along the reference trajectory. Then the
transverse part of \(\mathcal J_t\) is
\[
    P(\theta_t)A_\sigma^tP(\theta_0)^\top = P(\theta_t)A_\sigma^t
\]
From first order Taylor expansion, 
\[
\cos \theta = 1+ O(\theta^2), \qquad \sin \theta = \theta + O(\theta^3)
\]
hence,
\[
    P(\theta) = \begin{pmatrix}
    1&-\theta\\
    \theta&1
    \end{pmatrix} + O(\theta^2) = I + \begin{pmatrix}
    0&-1\\
    1&0
    \end{pmatrix}\theta + O(\theta^2).
\]
Let
\(
    s:=\sigma\tilde\eta,
\)
then
\begin{align*}
    A_\sigma & = B^{-1}\left( I + \begin{pmatrix}
    0&-1\\
    1&0
    \end{pmatrix}\left(-\frac{s}{24}\right) + O(s^2)\right)\left(B - \frac{s}{6}Q \right) \\
    & = I - \frac{s}{6}B^{-1}Q - B^{-1}\begin{pmatrix}
    0&-1\\
    1&0
    \end{pmatrix}\frac{s}{24}B + O(s^2).
\end{align*}
Let 
\[
    G=-\frac16B^{-1}Q-\frac1{24}B^{-1}\begin{pmatrix}
    0&-1\\
    1&0
    \end{pmatrix}B
\]
As \(s\downarrow0\),
\[
    A_\sigma=I+sG+O(s^2),
\]
A direct computation gives
\[
    G=
    \begin{pmatrix}
    -\frac{11}{180}&-\frac1{24}\\
    -\frac7{24}&-\frac{11}{60}
    \end{pmatrix},
    \qquad
    \det(G)=-\frac{41}{43200}<0.
\]
Thus \(G\) has one positive and one negative real eigenvalue. Standard
perturbation theory gives an expanding eigenvalue \(\rho>1\) of \(A_\sigma\)
with $\rho -1 = \Theta(s)$, thus since \(s=\sigma\tilde\eta\) and \(\tilde\eta\) is an absolute multiple of \(\eta\) there exist some absolute constants $c,C>0$ such that 
\[
    c\sigma \eta \leq \rho-1\leq C\sigma\eta.
\]
Since $\rho -1 \geq 0$ choosing $\eta_0$ to be small enough such that $\rho -1 \leq 1$ for any $\sigma \in (0,1]$ we also have
\[
    \frac{c\sigma\eta}{2}\leq\frac{\rho - 1}{2}\le \log\rho\le \rho - 1 \leq C\sigma\eta.
\]
Choosing $c_{\rho} = \frac{c}{2}$ and $C_{\rho} = C$ finished this part. Finally, since \(B\) has two distinct real eigenvalues, its eigenbasis is nondegenerate.
After decreasing the absolute constant \(\eta_0>0\), the parameter
\(s=\sigma\tilde\eta\) is uniformly small for all \(0<\sigma\le1\) and
\(0<\eta\le\eta_0\). Hence the eigenvectors of
\(A_\sigma=I+sG+O(s^2)\) are uniformly close to the eigenvectors of \(G\), and
the corresponding eigenbasis of \(A_\sigma\) is uniformly well-conditioned.

\subsubsection{Proof of Lemma~\ref{lem:second_derivative_update}}

The update satisfies
\[
    \nabla R(\Phi(y))
    =
    \nabla R(y)-\eta\nabla F(y).
\]
Bifferentiating once gives
\[
    \nabla^2R(\Phi(y))D\Phi(y)
    =
    \nabla^2R(y)-\eta\nabla^2F(y).
\]
Subtracting \(\nabla^2R(\Phi(y))\) from both sides gives
\[
    \nabla^2R(\Phi(y))(D\Phi(y)-I)
    =
    \nabla^2R(y)-\nabla^2R(\Phi(y))
    -
    \eta\nabla^2F(y).
\]
Since \(\nabla^2R\succeq I\),
\[
    \|(\nabla^2R(\Phi(y)))^{-1}\|\le1.
\]
Using the \(D^3R\) bound along the segment \([y,\Phi(y)]\subset U\), we get
\[
\begin{aligned}
    \|D\Phi(y)-I\|
    &\le
    \|\nabla^2R(y)-\nabla^2R(\Phi(y))\|
    +
    \eta\|\nabla^2F(y)\| \\
    &\le
    M_3\|\Phi(y)-y\|+\eta N_2 \\
    &\le
    M_3\Delta+\eta N_2.
\end{aligned}
\]

Differentiating the implicit equation twice in unit directions \(p,q\) on both sides gives
\[
\begin{aligned}
    \nabla^2R(\Phi(y))D^2\Phi(y)[p,q] +
    D^3R(\Phi(y))[D\Phi(y)p,D\Phi(y)q] =
    D^3R(y)[p,q]-\eta D^3F(y)[p,q].
\end{aligned}
\]
Using again \(\|(\nabla^2R(\Phi(y)))^{-1}\|\le1\), and adding and subtracting
\(D^3R(\Phi(y))[p,q]\), we obtain
\[
\begin{aligned}
    \|D^2\Phi(y)\|
    &\le
    \|D^3R(y)-D^3R(\Phi(y))\|
    +
    \eta\|D^3F(y)\| \\
    &\quad+
    \sup_{\|p\|=\|q\|=1}
    \left\|
    D^3R(\Phi(y))[D\Phi(y)p,D\Phi(y)q]
    -
    D^3R(\Phi(y))[p,q]
    \right\|.
\end{aligned}
\]
Since \([y,\Phi(y)]\subset U\), the \(D^4R\) bound gives
\[
    \|D^3R(y)-D^3R(\Phi(y))\|
    \le
    M_4\|\Phi(y)-y\|
    \le
    M_4\Delta.
\]
Let
\[
    E:=\|D\Phi(y)-I\|.
\]
For unit \(p,q\),
\[
    D\Phi(y)p=p+(D\Phi(y)-I)p,
    \qquad
    D\Phi(y)q=q+(D\Phi(y)-I)q.
\]
By bilinearity in the first two slots,
\[
\begin{aligned}
& D^3R(\Phi(y))[D\Phi(y)p,D\Phi(y)q]
    -
    D^3R(\Phi(y))[p,q] \\
&=
D^3R(\Phi(y))[(D\Phi(y)-I)p,q]
+
D^3R(\Phi(y))[p,(D\Phi(y)-I)q] \\
&\quad+
D^3R(\Phi(y))[(D\Phi(y)-I)p,(D\Phi(y)-I)q].
\end{aligned}
\]
Therefore
\[
    \left\|
    D^3R(\Phi(y))[D\Phi(y)p,D\Phi(y)q]
    -
    D^3R(\Phi(y))[p,q]
    \right\|
    \le
    2 M_3(E+E^2).
\]
Combining the estimates,
\[
    \|D^2\Phi(y)\|
    \le
    2\left(
        M_4\Delta
        +\eta N_3
        +M_3E
        +M_3E^2
    \right).
\]
Using
\[
    E\le M_3\Delta+\eta N_2,
\]
we conclude
\[
    \|D^2\Phi(y)\|
    \le
    2\Bigl(
        M_4\Delta
        +\eta N_3
        +M_3(M_3\Delta+\eta N_2)
        +M_3(M_3\Delta+\eta N_2)^2
    \Bigr).
\]

\subsubsection{Proof of Lemma~\ref{lem:nonlinear_endpoint_final}}

We first record the local well-definedness of the update. Since
\(
    \operatorname{dist}(w_t,U^c)\ge \frac{\gamma}{2}\omega^{-3/2},
\)
the assumption \(b\le c_b\omega^{-3/2}\), after decreasing \(c_b>0\), implies
\[
    B(w_t,Kb)\subset U,
    \qquad 0\le t<T.
\]
\cref{lem:quantitative_inverse_branch_endpoint_local} applied at \(w_{t+1}\), together with
\(\nabla^2R\succeq I\), gives a \(1\)-Lipschitz inverse branch of \(\nabla R\)
through \(w_{t+1}\) on a ball of radius comparable to \(\omega^{-3/2}\) around
\(\nabla R(w_{t+1})\).
By Lemma~\ref{lem:exact_endpoint_final},
\[
    \nabla R(w_{t+1})
    =
    \nabla R(w_t)-\eta\nabla F(w_t).
\]
Moreover, on \(U\),
\[
    \|\nabla^2R\|+\eta\|\nabla^2F\|\le C_0 + \eta .
\]
Hence for every \(y\in B(w_t,Kb)\) taking $\eta_0 \leq 1$,
\[
\begin{aligned}
    &\left\|
    \bigl(\nabla R(y)-\eta\nabla F(y)\bigr)
    -
    \bigl(\nabla R(w_t)-\eta\nabla F(w_t)\bigr)
    \right\|\lesssim Kb.
\end{aligned}
\]
Thus, after decreasing \(c_b\) further if necessary, \(\Phi\) is well-defined on
each ball \(B(w_t,Kb)\).
Next, from \cref{lem:quantitative_inverse_branch_endpoint_local}
\[
    \|\Phi(y)-y\|_2
    \le
    \|\nabla R(\Phi(y))-\nabla R(y)\|_2.
\]
Using
\[
    \nabla R(\Phi(y))-\nabla R(y)
    =
    -\eta\nabla F(y),
\]
we obtain
\[
    \|\Phi(y)-y\|_2
    \le
    \eta\|\nabla F(y)\|_2
    \le
    C_1\frac{1}{65}\sigma\eta .
\]

We now apply Lemma~\ref{lem:second_derivative_update}. By
Lemma~\ref{lem:regularity_endpoint_final}, on \(U\),
\[
    \|D^3R\|\le C_1\omega,
    \qquad
    \|D^4R\|\le C_1\omega^2,
\]
and
\[
    \|D^2F\|\le C_1\frac{1}{65}\sigma,
    \qquad
    \|D^3F\|\le C_1\frac{1}{65}\sigma\omega.
\]
Together with
\[
    \|\Phi(y)-y\|_2\le C_1\frac{1}{65}\sigma\eta,
\]
Lemma~\ref{lem:second_derivative_update} gives, on the relevant segments,  
\[
\begin{aligned}
    \|D^2\Phi(y)\|
    &\le
    2\Bigl(
        \omega^2\frac{1}{65}\sigma\eta
        +\eta\frac{1}{65}\sigma\omega
        +\omega(\omega\frac{1}{65}\sigma\eta+\eta\frac{1}{65}\sigma)        +\omega(\omega\frac{1}{65}\sigma\eta+\eta\frac{1}{65}\sigma)^2
    \Bigr) \\
    &\le
    16\sigma\eta\omega^2(1+\sigma\eta\omega)
\end{aligned}
\]

Now assume
\[
    [w_t,w'_t]\subset B(w_t,Kb).
\]
Taylor's theorem applied to \(\Phi\) on the segment \([w_t,w'_t]\) yields
\[
\begin{aligned}
    &\left\|
    \Phi(w'_t)-\Phi(w_t)
    -
    D\Phi(w_t)(w'_t-w_t)
    \right\|_2 \le
    16\sigma\eta\omega^2(1+\sigma\eta\omega)\|w'_t-w_t\|_2^2.
\end{aligned}
\]
Passing to the rotating frame the nonlinear error is exactly the Taylor remainder of the update map \(\Phi\) Indeed, since
\(w_{t+1}=\Phi(w_t)\) and \(w'_{t+1}=\Phi(w'_t)\), we have
\[
    Y_{t+1}
    =
    \mathcal T_{t+1}(w'_{t+1}-w_{t+1})
    =
    \mathcal T_{t+1}\bigl(\Phi(w'_t)-\Phi(w_t)\bigr).
\]
On the other hand, using
\[
    \mathcal T_{t+1}D\Phi(w_t)\mathcal T_t^{-1}=L, 
    \qquad
    Y_t=\mathcal T_t(w'_t-w_t),
\]
we get
\[
\begin{aligned}
    LY_t
    &=
    L\mathcal T_t(w'_t-w_t) \\
    &=
    \mathcal T_{t+1}D\Phi(w_t)\mathcal T_t^{-1}
    \mathcal T_t(w'_t-w_t) \\
    &=
    \mathcal T_{t+1}D\Phi(w_t)(w'_t-w_t).
\end{aligned}
\]
Therefore,
\[
\begin{aligned}
    Y_{t+1}-LY_t
    &=
    \mathcal T_{t+1}\bigl(\Phi(w'_t)-\Phi(w_t)\bigr)
    -
    \mathcal T_{t+1}D\Phi(w_t)(w'_t-w_t) \\
    &=
    \mathcal T_{t+1}
    \left[
        \Phi(w'_t)-\Phi(w_t)
        -
        D\Phi(w_t)(w'_t-w_t)
    \right].
\end{aligned}
\]
Now, since \(\mathcal T_t\) is an isometry in Euclidean norm and that
\(\|\cdot\|_\#\) is uniformly equivalent to the Euclidean norm, there exists an absolute constant $C_2 > 0$ such that
\[
    \|Y_{t+1}-LY_t\|_\#
    \le
    C_2\sigma\eta\omega^2(1+\sigma\eta\omega) \|Y_t\|_\#^2,
\]
which concludes the proof.

\subsubsection{Proof of Lemma~\ref{lem:quadratic_bootstrap}}\cref{lem:quadratic_bootstrap}
Let $t_*$ be the first index such that $q_{t_*}>2$, if such an index exists. Then $t_*\ge 1$ and $q_{t-1}\le 2$ for every $1\le t<t_*$. Summing the recursion up to time $t_*$ yields
\[
q_{t_*}
\le
1+\sum_{t=1}^{t_*} a_t q_{t-1}^2
\le
1+4\sum_{t=1}^{t_*} a_t
\le
2,
\]
a contradiction. Hence no such $t_*$ exists.

\section{Proofs for Section~\ref{sec:KL}}\label{sec:KL-upper-proof}
\subsection{Proof of Theorem~\ref{thm:kl-upper-bound}}
We first record two elementary stability properties of the reweighting map
that appears in entropy mirror descent.
\begin{lemma}\label{lem:stability-reweighting}
For any vector \(a\in\mathbb{R}_{>0}^d\) and \(\w,\w'\in\simp\), define
\[
    \Psi_a(\w)(i)
    =
    \frac{a(i)\w(i)}{\sum_{j=1}^d a(j)\w(j)}.
\]
The following hold:
\begin{enumerate}
    \item \(\|\Psi_a(\w)-\Psi_a(\w')\|_1
    \le
    \frac{\max_i a(i)}{\min_i a(i)}\|\w-\w'\|_1\).
    \item For any \(a'\in\mathbb{R}_{>0}^d\),
    \[
        \|\Psi_a(\w)-\Psi_{a'}(\w)\|_1
        \le
        \frac{\max_i a(i)/a'(i)}{\min_i a(i)/a'(i)}-1.
    \]
\end{enumerate}
\end{lemma}
\begin{proof}[Proof of \cref{lem:stability-reweighting}]
We prove the two claims separately. For the first claim, use
\[
    \sum_i \bigl(\w(i)-\min\{\w(i),\w'(i)\}\bigr)
    =
    \sum_i \bigl(\w'(i)-\min\{\w(i),\w'(i)\}\bigr)
    =
    \frac12\|\w-\w'\|_1.
\]
Let \(Z=\sum_i a(i)\w(i)\) and \(Z'=\sum_i a(i)\w'(i)\). Since
\[
    \sum_i \min\{\Psi_a(\w)(i),\Psi_a(\w')(i)\}
    \ge
    \frac{\sum_i a(i)\min\{\w(i),\w'(i)\}}{\max\{Z,Z'\}},
\]
and since \(\Psi_a(\w)\) and \(\Psi_a(\w')\) are probability vectors,
\[
    \|\Psi_a(\w)-\Psi_a(\w')\|_1
    \le
    2\,
    \frac{
    \max\{Z,Z'\}-\sum_i a(i)\min\{\w(i),\w'(i)\}
    }{\max\{Z,Z'\}}.
\]
The numerator is at most
\(\frac12(\max_i a(i))\|\w-\w'\|_1\), while the denominator is at least
\(\min_i a(i)\). This proves the first claim.

For the second claim, let \(r_i=a(i)/a'(i)\), \(r_{\min}=\min_i r_i\), and
\(r_{\max}=\max_i r_i\). For every \(i\),
\[
    \frac{\Psi_a(\w)(i)}{\Psi_{a'}(\w)(i)}
    =
    \frac{r_i}{\sum_j r_j a'(j)\w(j)/\sum_j a'(j)\w(j)}.
\]
The denominator on the right-hand side is a weighted average of the \(r_j\)'s,
and therefore lies in \([r_{\min},r_{\max}]\). Hence
\[
    \frac{r_{\min}}{r_{\max}}
    \le
    \frac{\Psi_a(\w)(i)}{\Psi_{a'}(\w)(i)}
    \le
    \frac{r_{\max}}{r_{\min}},
\]
which implies
\[
    \left|
    \frac{\Psi_a(\w)(i)}{\Psi_{a'}(\w)(i)}-1
    \right|
    \le
    \frac{r_{\max}}{r_{\min}}-1.
\]
Multiplying by \(\Psi_{a'}(\w)(i)\) and summing over \(i\) proves the second
claim.
\end{proof}
We are now ready to prove \cref{thm:kl-upper-bound}.
\begin{proof}[Proof of \cref{thm:kl-upper-bound}]
Fix \(w_0,w_0'\in\simp\), and let \(\{\w_t\}_{t=0}^T\) and
\(\{\w_t'\}_{t=0}^T\) be the corresponding MD trajectories. For each
\(t=0,\ldots,T-1\), define the coordinatewise exponentials
\[
    g_t=e^{-\eta\nabla F(\w_t)},
    \qquad
    g_t'=e^{-\eta\nabla F(\w_t')}.
\]
The negative-entropy MD update can be written as
\[
    \w_{t+1}=\Psi_{g_t}(\w_t),
    \qquad
    \w_{t+1}'=\Psi_{g_t'}(\w_t').
\]
Therefore,
\begin{align*}
    \|\w_{t+1}-\w_{t+1}'\|_1
    &\le
    \|\Psi_{g_t}(\w_t)-\Psi_{g_t}(\w_t')\|_1
    +
    \|\Psi_{g_t}(\w_t')-\Psi_{g_t'}(\w_t')\|_1 .
\end{align*}

We bound the two terms separately. By \cref{lem:stability-reweighting},
\[
    \|\Psi_{g_t}(\w_t)-\Psi_{g_t}(\w_t')\|_1
    \le
    \frac{\max_i g_t(i)}{\min_i g_t(i)}
    \|\w_t-\w_t'\|_1.
\]
For any \(i,j\), the tangent direction \(e_i-e_j\) has \(\ell_1\)-norm \(2\).
Since \(F\) is \(G\)-Lipschitz on the open simplex,
\(|\nabla F(\w_t)(i)-\nabla F(\w_t)(j)|\le2G\). Hence
\[
    \frac{\max_i g_t(i)}{\min_i g_t(i)}
    \le
    e^{2\eta G},
\]
and the first term is at most \(e^{2\eta G}\|\w_t-\w_t'\|_1\).

For the second term, another application of
\cref{lem:stability-reweighting} gives
\[
    \|\Psi_{g_t}(\w_t')-\Psi_{g_t'}(\w_t')\|_1
    \le
    \frac{\max_i g_t(i)/g_t'(i)}{\min_i g_t(i)/g_t'(i)}-1.
\]
By \(L\)-smoothness,
\[
    \max_i\bigl(\nabla F(\w_t)(i)-\nabla F(\w_t')(i)\bigr)
    -
    \min_i\bigl(\nabla F(\w_t)(i)-\nabla F(\w_t')(i)\bigr)
    \le
    2L\|\w_t-\w_t'\|_1.
\]
Therefore,
\[
    \|\Psi_{g_t}(\w_t')-\Psi_{g_t'}(\w_t')\|_1
    \le
    e^{2\eta L\|\w_t-\w_t'\|_1}-1.
\]
Since \(\|\w_t-\w_t'\|_1\le2\), the monotonicity of
\(z\mapsto(e^{2\eta Lz}-1)/z\) on \(z>0\), with continuity at \(z=0\), implies
\[
    e^{2\eta L\|\w_t-\w_t'\|_1}-1
    \le
    \frac{e^{4\eta L}-1}{2}\|\w_t-\w_t'\|_1.
\]
Combining the two bounds,
\begin{align*}
    \|\w_{t+1}-\w_{t+1}'\|_1
    &\le
    \left(e^{2\eta G}+\frac{e^{4\eta L}-1}{2}\right)
    \|\w_t-\w_t'\|_1 \\
    &\le
    e^{(2G+4L)\eta}\|\w_t-\w_t'\|_1.
\end{align*}
Iterating for \(T\) steps gives
\[
    \|\w_T-\w_T'\|_1
    \le
    e^{(2G+4L)\eta T}\|\w_0-\w_0'\|_1.
\]
Also, since both iterates remain in \(\simp\),
\(\|\w_T-\w_T'\|_1\le2\). Taking the supremum over all admissible
\(w_0'=w_0+p\) with \(\|p\|_1\le\varepsilon\) proves the theorem.
\end{proof}

\section{Proofs for Section~\ref{sec:stabilizing}}\label{sec:stabilizing-proof}
To prove the lower bounds for \cref{alg:anchored-MD-init,alg:anchored-MD-fixed}, We will use two standard facts about relative smoothness and relative strong
convexity. The first, adapted from \citet{attia2022uniform}, shows that adding
a multiple of the regularizer makes the objective relatively strongly convex.

\begin{lemma}
\label{lem:relative-attia}
Let \(F\) be convex and \(L\)-smooth with respect to \(\|\cdot\|\), and let
\(R\) be \(1\)-strongly convex with respect to \(\|\cdot\|\). Then, for every
\(\mu>0\), the function
\[
    F^\mu(\w) = F(\w)+\mu R(\w)
\]
is \((L+\mu)\)-smooth and \(\mu\)-strongly convex relative to \(R\).
\end{lemma}

The second, due to \citet{lu2018relatively}, gives the convergence guarantees
of MD under relative smoothness and relative strong convexity.

\begin{lemma}
\label{lem:relative-lu}
Let \(\Phi\) be \(L\)-smooth and \(\mu\)-strongly convex relative to
\(R\), and let
\(
    \w^\star \in \arg\min_{\w\in\dom}\Phi(\w).
\)
Then the trajectory \(\{\w_t\}_{t=0}^T\) generated by MD with $\eta = 1/L$ satisfies, for all \(t\ge1\),
\[
    D_R(\w^\star,\w_t)
    \le
    \left(1-\frac{\mu}{L}\right)^t
    D_R(\w^\star,\w_0),
\]
and
\[
    \Phi(\w_t)-\Phi(\w^\star)
    \le
    \frac{\mu D_R(\w^\star,\w_0)}
    {\left(1+\frac{\mu}{L-\mu}\right)^t-1}.
\]
\end{lemma}

We now prove \cref{thm:stabilizing-alg-init,thm:stabilizing-alg-fixed}.

\subsection{Proof of Theorem~\ref{thm:stabilizing-alg-init}} 
Fix two initializations \(w_0,w_0'\in\dom\). Let
$\{\w_t\}_{t=0}^T$ be the trajectory of \cref{alg:anchored-MD-init} initialized
at \(w_0\), and let $\{\w_t'\}_{t=0}^T$ be the trajectory initialized at \(w_0'\).
Thus each trajectory is anchored at its own initialization:
\[
    F_{w_0}^\mu(\w)=F(\w)+\mu D_R(\w,w_0),
    \qquad
    F_{{w_0}'}^\mu(\w)=F(\w)+\mu D_R(\w,{w_0}').
\]
Since
\[
    \nabla F_{w_0}^\mu(\w)
    =
    \nabla F(\w)+\mu\bigl(\nabla R(\w)-\nabla R({w_0})\bigr),
\]
\cref{alg:anchored-MD-init} is MD applied to the regularized objective $F^\mu_{w_0}$ with $\eta = 1/(L+\mu)$. First we prove that $F^{\mu}_{w_0}$ is $\mu$ strongly-convex and $L+\mu$-smooth relative to $R$. For the proof we use \cref{lem:relative-attia}. Indeed,
\[
    F^\mu_{w_0}(\w)
    =
    F(\w)
    +
    \mu R(\w)
    -
    \mu R({w_0})
    -
    \mu\langle \nabla R({w_0}),\w-u\rangle .
\]
The last two terms are affine in \(\w\), and thus do not affect relative
smoothness or relative strong convexity. Hence \(F^\mu_{w_0}\) is \((\mu+L)\)-smooth and \(\mu\)-strongly convex relative to \(R\). The same holds for $w_0'$.  We may therefore apply \cref{lem:relative-lu} to these two cases.
For each anchor \(v\in\dom\), let
\(
    w^\star_{\mu,v}\in\arg\min_{w\in\dom}F^\mu_v(w).
\)
We first bound the movement of these regularized minimizers as the anchor
changes. By the first-order optimality conditions for \(w^\star_{\mu,{w_0}}\) and
\(w^\star_{\mu,{w_0}'}\), applied with the other minimizer as a comparison point,
\begin{align*}
&\left\langle
    \nabla F(w^\star_{\mu,{w_0}})
    +
    \mu\left(\nabla R(w^\star_{\mu,{w_0}})-\nabla R({w_0})\right),
    w^\star_{\mu,{w_0}'}-w^\star_{\mu,{w_0}}
\right\rangle
\ge0,\\
&\left\langle
    \nabla F(w^\star_{\mu,{w_0}'})
    +
    \mu\left(\nabla R(w^\star_{\mu,{w_0}'})-\nabla R({w_0}')\right),
    w^\star_{\mu,{w_0}}-w^\star_{\mu,{w_0}'}
\right\rangle
\ge0.
\end{align*}
Adding the two inequalities and using the monotonicity of \(\nabla F\) gives
\[
    \left\langle
        \nabla R(w^\star_{\mu,{w_0}})-\nabla R(w^\star_{\mu,{w_0}'}),
        w^\star_{\mu,{w_0}}-w^\star_{\mu,{w_0}'}
    \right\rangle
    \le
    \left\langle
        \nabla R({w_0})-\nabla R({w_0}'),
        w^\star_{\mu,{w_0}}-w^\star_{\mu,{w_0}'}
    \right\rangle .
\]
Since \(R\) is \(1\)-strongly convex,
\[
    \|w^\star_{\mu,{w_0}}-w^\star_{\mu,{w_0}'}\|
    \le
    \|\nabla R({w_0})-\nabla R({w_0}')\|_* .
\]

Next we bound the distance from each trajectory to its own regularized
minimizer. Since \(F^\mu_{w_0}(w^\star_{\mu,{w_0}})\le F^\mu_{w_0}({w_0})=F({w_0})\), we have
\[
    \mu D_R(w^\star_{\mu,{w_0}},{w_0})
    \le
    F({w_0})-F(w^\star_{\mu,{w_0}})
    \le
    GD.
\]
By \cref{lem:relative-lu},
\[
    D_R(w^\star_{\mu,{w_0}},\w_T)
    \le
    \left(1+\frac{\mu}{L}\right)^{-T}
    D_R(w^\star_{\mu,{w_0}},{w_0})
    \le
    \left(1+\frac{\mu}{L}\right)^{-T}
    \frac{GD}{\mu}.
\]
The same bound holds for the trajectory initialized at \(u'\). Using strong
convexity of \(R\) and the triangle inequality,
\begin{align*}
    \|\w_T-\w_T'\|
    &\le
    \|\w_T-w^\star_{\mu,{w_0}}\|
    +
    \|w^\star_{\mu,{w_0}}-w^\star_{\mu,{w_0}'}\|
    +
    \|w^\star_{\mu,{w_0}'}-\w_T'\|\\
    &\le
    \|\nabla R({w_0})-\nabla R({w_0}')\|_*
    +
    2\sqrt{2\left(1+\frac{\mu}{L}\right)^{-T}\frac{GD}{\mu}}.
\end{align*}
With \(\mu=8L\log T/T\), the elementary bound
\[
    \left(1+\frac{8\log T}{T}\right)^T\ge T^2
    \qquad (T\ge2)
\]
implies
\[
    \|\w_T-\w_T'\|
    \le
    \|\nabla R({w_0})-\nabla R({w_0}')\|_*
    +
    \sqrt{\frac{GD}{L T\log T}}.
\]
Letting $w_0' = w_0 + p$ and taking the supremum over all admissible perturbations \(p\) proves the first stability bound. If \(R\) is
\(\beta\)-smooth, then
\[
    \|\nabla R(w_0)-\nabla R(w_0+p)\|_*
    \le
    \beta\|p\|
    \le
    \beta\varepsilon,
\]
which gives the stated smooth-regularizer form.

It remains to prove the optimization guarantee for the trajectory initialized at \(w_0\). Let
\(
    w^\star_{\mu,w_0}\in\arg\min_{w\in\dom}F^\mu_{w_0}(w).
\)
\begin{align*}
    F(\w_T) - F(\w^\star)
    & = F^{\mu}_{w_0}(\w_T) - F^{\mu}_{w_0}(\w^{\star}) + \mu(D_R(\w^\star, w_0) - D_R(\w_T, w_0))\\
    & \leq F^{\mu}_{w_0}(\w_T) - F^{\mu}_{w_0}(\w^{\star}_{\mu,w_0}) + \mu D_R(\w^\star, w_0).
\end{align*}
From \cref{lem:relative-lu},
\begin{align*}
    F^{\mu}_{w_0}(\w_T) - F^{\mu}_{w_0}(\w^\star_{\mu,w_0})
    \leq
    \frac{\mu D_R(\w^\star_{\mu,w_0},w_0)}{\left(1 + \frac{\mu}{L}\right)^T -1}
    \leq
    \frac{\mu D_R(\w^\star_{\mu,w_0},w_0)}{T^2-1}.
\end{align*}
The last inequality follows from the choice of $\mu$ and
$\left(1+\frac{8\log T}{T}\right)^T\ge T^2$. Overall,
\begin{align*}
    F(\w_T) - F(\w^\star)
    & \leq
    \frac{\mu D_R(\w^\star_{\mu,w_0},w_0)}{T^2-1}
    +
    \mu D_R(\w^\star, w_0).
\end{align*}
Finally, as above, \(F^\mu_{w_0}(w^\star_{\mu,w_0})\le F^\mu_{w_0}(w_0)=F(w_0)\), so
\[
    \mu D_R(w^\star_{\mu,w_0},w_0)
    \le
    F(w_0)-F(w^\star_{\mu,w_0})
    \le
    GD.
\]
Substituting this and \(\mu=8L\log T/T\), we obtain
\[
    F(\w_T)-F(w^\star)
    \le
    \frac{GD}{T^2-1}
    +
    \frac{8L\log T}{T}D_R(w^\star,w_0).
\]
Since \(T\ge2\), we have \((T^2-1)^{-1}\le 2/T^2\), and therefore
\[
    F(\w_T)-F(w^\star)
    \le
    \frac{8L\log T}{T}D_R(w^\star,w_0)
    +
    \frac{2GD}{T^2}.
\]

\subsection{Proof of Theorem~\ref{thm:anchored-kl-lower}}\label{sec:ill-conditioned-anchor}

Let
\[
    i_m\in\arg\min_{i\in[d]}\w_0(i),
    \qquad
    \w_0(i_m)=w_0^{\min}.
\]
Since \(d\ge2\), we have \(w_0^{\min}\le 1/2\), and therefore
\[
    1-w_0^{\min}\ge \frac12> \frac{\varepsilon}{2}.
\]
Define \(p\in\mathbb{R}^d\) by
\[
    p(i_m)=\frac{\varepsilon}{2},
    \qquad
    p(i)=
    -\frac{\varepsilon}{2} \cdot \frac{\w_0(i)}{1-w_0^{\min}}
    \quad\text{for }i\neq i_m.
\]
Then \(\sum_{i=1}^d p(i)=0\). Moreover, since
\(\frac{\varepsilon}{2}< 1-w_0^{\min}\), for every \(i\neq i_m\),
\[
    \w_0(i)+p(i)
    =
    \w_0(i)\left(
    1-\frac{\varepsilon}{2}\cdot \frac{1}{1-w_0^{\min}}
    \right)
    >
    0.
\]
Also,
\[
    \w_0(i_m)+p(i_m)
    =
    w_0^{\min}+\frac{\varepsilon}{2}
    >
    0.
\]
Therefore
\[
    \w_0':=\w_0+p\in\simp.
\]
Furthermore,
\[
    \|p\|_1
    =
    \frac{\varepsilon}{2}
    +
    \sum_{i\neq i_m}
    \frac{\varepsilon}{2}\cdot \frac{\w_0(i)}{1-w_0^{\min}}
    = \varepsilon.
\]

Now define
\[
    K:=
    \min\left\{
    \exp\left(
        \frac{1}{\mu}
        \left[
        1-\left(\frac{L}{L+\mu}\right)^T
        \right]
    \right),
    \frac{1}{w_0^{\min}},
    \frac{1}{\varepsilon}
    \right\}.
\]
Let \(F(\w)=\langle g,\w\rangle\), where
\[
    g(i)=
    \begin{cases}
    -\dfrac{\log K}{
        \frac{1}{\mu}
        \left[
        1-\left(\frac{L}{L+\mu}\right)^T
        \right]
    }, & i=i_m,\\[1.2em]
    0, & i\neq i_m.
    \end{cases}
\]
By the definition of \(K\),
\[
    0
    \le
    \frac{\log K}{
        \frac{1}{\mu}
        \left[
        1-\left(\frac{L}{L+\mu}\right)^T
        \right]
    }
    \le 1.
\]
Hence \(\|g\|_\infty\le1\), so \(F\) is \(1\)-Lipschitz with respect to
\(\|\cdot\|_1\). 

We now compute the dynamics of \cref{alg:anchored-MD-init} for this linear
objective. Fix an initialization \(w_0\in\simp\), and let
\(\{\w_t\}_{t=0}^T\) be the trajectory of the algorithm initialized at \(w_0\), from symmetry the same proof will follow for $w_0'$. 
Since \(F\) is linear, \(\nabla F(\w_t)=g\) for every \(t\), and the update is
\[
    \w_{t+1}
    =
    \arg\min_{\w\in\Delta_d}
    \left\{
    \left\langle
    g+\mu(\nabla R(\w_t)-\nabla R(w_0)),
    \w-\w_t
    \right\rangle
    +
    (\mu+L)D_R(\w,\w_t)
    \right\}.
\]
We first show that the iterates remain in the interior of the simplex. We prove by induction on $t$. For the base case indeed $w_0 \in \simp$, now suppose
that \(\w_t\in\simp\), and consider the objective minimized
in the update. Let \(\bar w\in \Delta_d\) be a minimizer. We claim that
\(\bar w\in\simp\). Suppose toward contradiction that
\(\bar w(i)=0\) for some coordinate \(i\). Since \(\bar w\in\Delta_d\), there
exists a coordinate \(j\) such that \(\bar w(j)>0\). For sufficiently small
\(\alpha>0\), the point
\[
    \bar w^\alpha=\bar w+\alpha e_i-\alpha e_j
\]
belongs to \(\Delta_d\). The directional derivative of the Bregman term in this
direction is
\[
    \left\langle
    \nabla D_R(\w,\w_t), e_i-e_j
    \right\rangle
    =
    \log\frac{\w(i)}{\w_t(i)}
    -
    \log\frac{\w(j)}{\w_t(j)}.
\]
As \(\w(i)\to0\), this quantity tends to \(-\infty\), whereas the
directional derivative of the linear part of the update, \(
\left\langle
g, e_i-e_j
\right\rangle
\)
is finite. Hence, for
all sufficiently small \(\alpha>0\), moving from \(\bar w\) to
\(\bar w^\alpha\) strictly decreases the objective, contradicting the
optimality of \(\bar w\). Therefore the minimizer cannot place zero mass on any
coordinate, and so \(\w_{t+1}\in\simp\). This proves the claim.  

Hence from KKT conditions, there exists a scalar \(\lambda_t\) such that, for every coordinate \(i\),
\[
    g(i)
    +
    \mu\bigl(\nabla R(\w_t)(i)-\nabla R(w_0)(i)\bigr)
    +
    (\mu+L)\bigl(\nabla R(\w_{t+1})(i)-\nabla R(\w_t)(i)\bigr)
    +
    \lambda_t
    =
    0.
\]
Subtracting the conditions for coordinates \(i\) and \(j\), the scalar
\(\lambda_t\) cancels. Since for negative entropy
\[
    \nabla R(\w)(i)-\nabla R(\w)(j)
    =
    \log\frac{\w(i)}{\w(j)},
\]
we obtain
\[
    \log\frac{\w_{t+1}(i)}{\w_{t+1}(j)}
    =
    \frac{L}{L+\mu}
    \log\frac{\w_t(i)}{\w_t(j)}
    +
    \frac{\mu}{L+\mu}
    \log\frac{w_0(i)}{w_0(j)}
    -
    \frac{1}{L+\mu}
    \bigl(g(i)-g(j)\bigr).
\]
Unfolding the recursion gives
\[
    \log\frac{\w_T(i)}{\w_T(j)}
    =
    \log\frac{w_0(i)}{w_0(j)}
    -
    \frac{1}{\mu}
    \left[
    1-\left(\frac{L}{L+\mu}\right)^T
    \right]
    \bigl(g(i)-g(j)\bigr).
\]
Equivalently,
\[
    \w_T(i)
    =
    \frac{
    w_0(i)
    \exp\left(
    -
    \frac{1}{\mu}
    \left[
    1-\left(\frac{L}{L+\mu}\right)^T
    \right]
    g(i)
    \right)
    }{
    \sum_{j=1}^d
    w_0(j)
    \exp\left(
    -
    \frac{1}{\mu}
    \left[
    1-\left(\frac{L}{L+\mu}\right)^T
    \right]
    g(j)
    \right)
    }.
\]

Applying this formula first with \(\w_0\), and then with
\(\w_0'\), we get
\[
    \w_T(i_m)
    =
    \frac{K w_0^{\min}}{
    1+(K-1)w_0^{\min}
    },
\]
and
\[
    \w_T'(i_m)
    =
    \frac{K(w_0^{\min}+\frac{\varepsilon}{2})}{
    1+(K-1)(w_0^{\min}+\frac{\varepsilon}{2})
    }.
\]
Hence
\[
    \|\w_T-\w_T'\|_1
    \ge
    2\left(\w_T'(i_m)-\w_T(i_m)\right).
\]
By direct computation,
\[
\begin{aligned}
    2\left(\w_T'(i_m)-\w_T(i_m)\right)
    &=
    2\left(
    \frac{K(w_0^{\min}+\frac{\varepsilon}{2})}{
    1+(K-1)(w_0^{\min}+\frac{\varepsilon}{2})}
    -
    \frac{K w_0^{\min}}{
    1+(K-1)w_0^{\min}}
    \right) \\
    &=
    \frac{K\varepsilon}{
    \left(1+(K-1)(w_0^{\min}+\frac{\varepsilon}{2})\right)
    \left(1+(K-1)w_0^{\min}\right)
    }.
\end{aligned}
\]
We now bound the denominator. Since \(K\le 1/w_0^{\min}\),
\[
    1+(K-1)w_0^{\min}
    \le
    1+K w_0^{\min}
    \le
    2.
\]
Also, since \(K\le 1/w_0^{\min}\) and \(K\le1/\varepsilon\),
\[
    1+(K-1)(w_0^{\min}+\frac{\varepsilon}{2})
    \le
    1+K w_0^{\min}+\frac{K\varepsilon}{2}
    \le
    1+1+\frac12
    =
    \frac52.
\]
Thus
\[
    \|\w_T-\w_T'\|_1
    \ge
    \frac{K\varepsilon}{5}.
\]

By the definition of \(K\),
\[
    K\varepsilon
    =
    \min\left\{
    \varepsilon
    \exp\left(
        \frac{1}{\mu}
        \left[
        1-\left(\frac{L}{L+\mu}\right)^T
        \right]
    \right),
    \frac{\varepsilon}{w_0^{\min}},
    1
    \right\}.
\]
Therefore
\[
    \delta_A(\w_0,\varepsilon)
    \ge
    \frac{1}{5}
    \min\left\{
    1,\,
    \varepsilon
    \exp\left(
        \frac{1}{\mu}
        \left[
        1-\left(\frac{L}{L+\mu}\right)^T
        \right]
    \right),
    \frac{\varepsilon}{w_0^{\min}}
    \right\}.
\]

It remains to prove the simpler lower bound. Since
\[
    \left(\frac{L}{L+\mu}\right)^T
    =
    \left(1-\frac{\mu}{L+\mu}\right)^T
    \le
    \exp\left(-\frac{\mu T}{L+\mu}\right),
\]
we have
\[
    \frac{1}{\mu}
    \left[
    1-\left(\frac{L}{L+\mu}\right)^T
    \right]
    \ge
    \frac{1}{\mu}
    \left[
    1-\exp\left(-\frac{\mu T}{L+\mu}\right)
    \right].
\]
Using the elementary inequality
\[
    1-e^{-x}
    \ge
    (1-e^{-1})\min\{x,1\},
    \qquad x\ge0,
\]
with \(x=\frac{\mu T}{L+\mu}\), we get
\[
    \frac{1}{\mu}
    \left[
    1-\left(\frac{L}{L+\mu}\right)^T
    \right] \geq \frac{1}{\mu}
    \left[
    1-\exp\left(-\frac{\mu T}{L+\mu}\right)
    \right]
    \ge
    (1-e^{-1})
    \min\left\{
    \frac{T}{L+\mu},
    \frac{1}{\mu}
    \right\}.
\]
Plugging this into the exact lower bound gives
\[
    \delta_A(\w_0,\varepsilon)
    \ge
    \frac{1}{5}
    \min\left\{
    1,\,
    \varepsilon
    \exp\left(
        (1-e^{-1})
        \min\left\{
        \frac{T}{L+\mu},
        \frac{1}{\mu}
        \right\}
    \right),
    \frac{\varepsilon}{w_0^{\min}}
    \right\}.
\]
If additionally $d \geq 1/\epsilon$ then $w_0^{\min} \leq 1/d$ hence $\epsilon/w_0^{\min} \geq 1$ and the bound becomes,
\[
    \delta_A(\w_0,\varepsilon)
    \ge
    \frac{1}{5}
    \min\left\{
    1,\,
    \varepsilon
    \exp\left(
        (1-e^{-1})
        \min\left\{
        \frac{T}{L+\mu},
        \frac{1}{\mu}
        \right\}
    \right)
    \right\}.
\]
Finally, using the same arguments as in the proof of \cref{thm:kl-lower}, we have 
\(
    \min\left\{1,\frac{\varepsilon}{w_0^{\min}}\right\}
    \ge
    \frac12
    \min\{1,\varepsilon\initbeta\},
\)
which concludes the proof.

\subsection{Proof of Theorem~\ref{thm:stabilizing-alg-fixed}}

Let
$\{\w_t\}_{t=0}^T$ be the trajectory of \cref{alg:anchored-MD-fixed} initialized at \(w_0\), and let $\{\w_t'\}_{t=0}^T$ be the trajectory initialized at \(w_0'\). Denote $F^{\mu}(\w) = F(\w) + \mu D_R(w,w_a) $. It is easy to see that \cref{alg:anchored-MD-fixed} simply runs MD updates on $F^{\mu}$.  
As a corollary of \cref{lem:relative-attia}, since Bregman divergence is unaffected by the addition of affine functions, \(F^\mu\) is \((\mu+L)\)-smooth and \(\mu\)-strongly convex relative to \(R\). We may therefore apply \cref{lem:relative-lu}. We will start with initialization stability. The following holds:
\begin{align*}
    \|\w_T - \w_T'\|^2 & \leq 2\| \w_T - \w^\star_{\mu}\|^2 + 2\|\w_T' - \w^\star_{\mu}\|^2 \\
    & \leq 4D_R( \w^\star_{\mu}, \w_T) + 4D_R(\w^\star_{\mu}, \w_T') \tag{$R$ is $1$-strongly convex} \\
    & \leq 4\left(1- \frac{\mu}{\mu + L} \right)^T \left( D_R(\w^\star_{\mu},\w_0)+ D_R(\w^\star_{\mu},\w_0') \right) \tag{\cref{lem:relative-lu}} \\
    & \leq 4\left( 1+ \frac{\mu}{L} \right)^{-T} \cdot \left( D_R(\w^\star_{\mu},\w_0)+ D_R(\w^\star_{\mu},\w_0') \right).
\end{align*} 
Plugging in $\mu = \frac{8L \log T}{T}$, using $(1+\frac{8\log T}{T})^T \geq T^2$ for $T\geq 2$, we have
\begin{align*}
    \|\w_T - \w_T'\|_1^2 &\leq  4\left( D_R(\w^\star_{\mu},\w_0)+ D_R(\w^\star_{\mu},\w_0') \right) \left(1+ \frac{8\log T}{T}\right)^{-T} \\
    & \leq \frac{4}{T^2}\left( D_R(\w^\star_{\mu},\w_0)+ D_R(\w^\star_{\mu},\w_0') \right) 
\end{align*}
Regarding optimization we have the following, 
\begin{align*}
    F(\w_T) - F(\w^\star) & = F^{\mu}(\w_T) - F^{\mu}(\w^{\star}) + \mu(D_R(\w^\star,w_a) - D_R(\w_T,w_a))\\
    & \leq F^{\mu}(\w_T) - F^{\mu}(\w^{\star}_{\mu}) +\mu D_R(\w^\star,w_a) 
\end{align*}
From \cref{lem:relative-lu} and the choice of $\mu$
\begin{align*}
    F(\w_T) - F(\w^\star)
    &\leq
    \frac{\mu D_R(\w^\star_{\mu},\w_0)}
    {\left(1 + \frac{\mu}{L}\right)^T -1}
    + \mu D_R(\w^\star,w_a) \\
    &\leq
    \frac{8L \log T \cdot D_R(\w^\star_{\mu},\w_0)}
    {T(T^2 - 1)}
    +
    \frac{8L \log T  D_R(\w^\star,w_a)}{T}.
\end{align*}
which concludes the proof.

\end{document}